\documentclass[11pt,a4paper]{article}
\usepackage[utf8]{inputenc}
\usepackage{multicol,amsmath,amsfonts,amssymb,amsthm,graphicx}
\usepackage[hmargin=2cm,vmargin=2cm]{geometry}
\usepackage{cite}
\usepackage[colorlinks=true,urlcolor=blue,citecolor=blue,anchorcolor=blue,linkcolor=blue]{hyperref}
\usepackage[linesnumbered,boxed]{algorithm2e}
\DontPrintSemicolon

\newtheorem{lm}{Lemma}

\newtheorem{crl}{Corollary}

\newtheorem{thm}{Theorem}
\usepackage{enumerate,etoolbox,paralist,url}
\usepackage{authblk}
\title{On the Theoretical Capacity of Evolution Strategies\\ to Statistically Learn the Landscape Hessian}
\author{Ofer M. Shir}
	\affil{School of Computer Science, Tel-Hai College, and	The Galilee Research Institute - Migal,	Upper Galilee, Israel\\
    \href{mailto:ofersh@telhai.ac.il}{ofersh@telhai.ac.il} }
\author{Jonathan Roslund}    
	\affil {Laboratoire Kastler Brossel, Universit\'{e} Pierre et Marie Curie, Paris, France\\
	\href{mailto:jroslund@lkb.upmc.fr}{jroslund@lkb.upmc.fr} }
\author{Amir Yehudayoff}
	\affil{Department of Mathematics, Technion - Israel Institute of Technology, Haifa, Israel\\
	\href{mailto:amir.yehudayoff@gmail.com}{amir.yehudayoff@gmail.com} }
\date{}
\begin{document}
\maketitle
\begin{abstract}
We study the theoretical capacity to statistically learn local landscape information by Evolution Strategies (ESs). Specifically, we investigate the covariance matrix when constructed by ESs operating with the selection operator alone. We model continuous generation of candidate solutions about quadratic basins of attraction, with deterministic selection of the decision vectors that minimize the objective function values. Our goal is to rigorously show that accumulation of winning individuals carries the potential to reveal valuable information about the search landscape, e.g., as already practically utilized by derandomized ES variants. We first show that the statistically-constructed covariance matrix over such winning decision vectors shares the same eigenvectors with the Hessian matrix about the optimum. We then provide an analytic approximation of this covariance matrix for a non-elitist multi-child $(1,\lambda)$-strategy, which holds for a large population size $\lambda$. Finally, we also numerically corroborate our results.
\end{abstract}
\begin{center}
\end{center}
\section{Introduction}
ESs \cite{Beyer-Schwefel}, popular heuristics that excel in global optimization of continuous search landscapes, utilize a Gaussian-based update (variation) step with an evolving covariance matrix.
Since the development of ESs, it has been assumed that this learned covariance matrix, which defines the variation operation, approximates the inverse Hessian of the search landscape. It was supported by the rationale that locating the global optimum by an ES can be accommodated using mutation steps that fit the actual landscape, or in other words, that the optimal covariance distribution can offer mutation steps whose \emph{equidensity probability contours} match the \emph{level sets} of the landscape (maximizing the progress rate at the same time) \cite{Rudolph92}. 

It has been additionally argued that reducing a general problem to an \emph{isotropic quadratic} problem may be achieved by sampling search-points based upon a covariance matrix that is the inverse Hessian (in equivalence to replacing the Euclidean distance measure with the Mahalanobis metric). Since ESs operate well on such isotropically quadratic problems, a successful ES run suggests that learning the covariance was accomplished. 
Nevertheless, it has never been formally proven that ESs' machinery can indeed learn the inverse of the Hessian. Rudolph \cite{Rudolph92} showed that ESs are capable of facilitating such learning and derived practical bounds on the population size toward the end of a successful learning period. That study paved the way toward accumulation of past search information by ESs by means of covariance matrices or any other forms of statistically learned algebraic structures. Especially, accumulation of selected individuals is practically utilized by derandomized ES variants \cite{Hansen01completely}, and it has led to the formulation of a successful family of search heuristics \cite{Baeck2013contemporary}.

The goal of the current study is to investigate the statistical learning potential of an accumulated set of selected individuals (the so-called \textit{winners} of each generation) \textbf{when the ES operates in the vicinity of a landscape optimum}. 
We argue and prove that accumulation of such \textit{winning} individuals carries the potential to reveal valuable search landscape information. In particular, we consider the statistically-constructed covariance matrix over \textit{winning} decision vectors and prove that it \textit{commutes} with the Hessian matrix about the optimum (i.e., the two matrices share the same eigenvectors, and therefore their level sets are positioned along the same axes). 
This result indicates that in learning a covariance, an ES deduces the sensitive directions for an effective optimization.
This carries the potential of a great benefit to an ES that learns a covariance matrix and may deduce the 
sensitive directions for effective optimization.
Furthermore, we provide an analytic approximation of the covariance matrix, which holds for a large population size and when the Hessian is well-behaved (formal details appear below).

The remainder of this paper is organized as follows.
The problem is formally stated in Section \ref{sec:problem}, where the assumed model is described in detail.
In Section \ref{sec:covariance} we formulate the covariance matrix, derive the necessary density function, and then prove that the covariance and the Hessian commute. Section \ref{sec:solving_singlewinner} provides an analytical covariance approximation for the problem of a $(1,\lambda)$-ES. A simulation study encompassing various landscape scenarios for $(1,\lambda)$ selection is presented in Section \ref{sec:simulation}, constituting a numerical corroboration for the theoretical outcomes in Sections \ref{sec:covariance} and \ref{sec:solving_singlewinner}. Finally, the results are discussed in Section \ref{sec:discussions}.

\section{Statistical Landscape Learning}\label{sec:problem}
We outline the \textit{research question} that we target:
\begin{quote}
What is the relation between the statistically-learned covariance matrix and the landscape Hessian if a single winner is selected in each iteration assuming generated samples that follow an isotropic Gaussian (no adaptation)?
%
\end{quote}
We focus on the \textit{a posteriori} statistical construction of the covariance matrix of the decision variables upon reaching the proximity of the optimum, when subject to ES operation.

In what follows, we formulate the problem, assume a model and present our notation.
\subsection{Problem Statement and Assumed Model}
Let ${J}: \mathbb{R}^n \to \mathbb{R}$ denote the objective function subject to minimization.
We assume that $J$ is minimized at the location $\vec{x}^{*}$,
which is assumed for simplicity to be the origin.
The objective function may be {\em Taylor-expanded} about the optimum.
We model the $n$-dimensional basin of attraction about $\vec{x}^{*}$ by means of a quadratic approximation.
We assume that this expansion is precise
\begin{equation}
\displaystyle 
J\left(\vec{x}-\vec{x}^{*}\right)\ =
J\left(\vec{x}\right)\ = \vec{x}^T \cdot \mathcal{H} \cdot \vec{x},
\end{equation}
with $\mathcal{H}$ being the landscape Hessian about the optimum.

The canonical non-elitist single-parent ES search process operates in the following manner:
The ES generates $\lambda$ search-points $\vec{x}_1,\ldots,\vec{x}_\lambda$
in each iteration, based upon Gaussian sampling with respect to the given search-point.
We are especially concerned with the canonical variation operator of ES, which adds a normally distributed \textit{mutation}
$\vec{z} \sim \mathcal{N} \left(\vec{0},\mathbf{I}\right)$.
That is, $\vec{x}_1,\ldots,\vec{x}_\lambda$ are independent and each is $\mathcal{N} \left(\vec{0},\mathbf{I}\right)$.
Following the evaluation of those $\lambda$ search points with respect to $J\left(\vec{x}\right)$, the best (minimal) individual is selected and recorded as
\begin{equation}
\vec{y} = \arg\min \{ J(\vec{x}_i)\}.
\end{equation}
Finally, let $\omega$ denote the \textit{winning} objective function value,
\begin{equation}
\omega = J(\vec{y}) = \min \left\{ J_{1},~ J_{2},~ \ldots ,~ J_{\lambda} \right\},
\end{equation}
where $J_i = J(\vec{x}_i)$.

We mention the difference between the optimization phase, which aims to arrive at the optimum and is not discussed here, to the statistical learning of the basin -- which lies in the focus of this study.

The sampling procedure is summarized as Algorithm 1, wherein \texttt{statCovariance} refers to a routine for \textit{statistically} constructing a covariance matrix from raw observations.
\IncMargin{1em}
\RestyleAlgo{boxed} 
\begin{algorithm}
\SetKwInOut{Output}{output}
\caption{Statistical sampling by $(1,\lambda)$-selection}
$t \leftarrow 0$\;
$\mathcal{S} \leftarrow$ $\emptyset$\;
\Repeat{$t \geq N_{\texttt{iter}}$} {
 \For{$k\leftarrow 1$ \KwTo $\lambda$}{
 $\vec{x}^{(t+1)}_{k} \leftarrow \vec{x}^{*} + \vec{z}_k,~~~\vec{z}_k \sim \mathcal{N} \left( \vec{0},\mathbf{I} \right)$\;
 $J^{(t+1)}_{k}\leftarrow$ \texttt{evaluate} $\left(\vec{x}^{(t+1)}_{k}\right)$\;
 } 
 $m_{t+1} \leftarrow\arg\min \left(\left\{J^{(t+1)}_{\imath}\right\}_{\imath=1}^{\lambda} \right)$\;
 $\mathcal{S} \leftarrow \mathcal{S} \cup \left\{ \vec{x}^{(t+1)}_{m_{t+1}} \right\} $\;
 $t \leftarrow t+1$\;
 }
\Output{$\mathcal{C}^{\texttt{stat}}=$\texttt{statCovariance}$\left(\mathcal{S}\right)$}
\label{algo:ES_sampling}
\end{algorithm}
\DecMargin{1em}

\subsection{Probability Functions of the ES Step}
The length of a mutation vector, $\sqrt{\vec{z}^T\vec{z}}$, obeys the so-called $\chi$-distribution with $n$ degrees of freedom.
Upon assuming a \textit{quadratic} basin of attraction, $\psi=J(\vec{z})$ is a random variable which obeys a generalized $\chi^2$-distribution. We consider two cases:
\begin{compactenum}
\item The basic, simplified case of an isotropic basin, that is, its Hessian matrix constitutes the identity: $\mathcal{H}=\mathbf{I}$. 
In this case, the distribution of $\psi$ is the standard $\chi^2$-distribution, possessing the following cumulative distribution function (\textit{CDF}) accounting for the search-space dimensionality $n$:
\begin{equation}\label{eq:cdf_chi}
\displaystyle F_{\chi^2}\left(\psi \right) = \frac{1}{2^{n/2}\Gamma\left(n/2\right)} \int_0^{\psi}t^{\frac{n}{2}-1}\exp\left(-\frac{t}{2} \right) ~ \textrm{d}t
\end{equation}
with $\Gamma(t)$ being the Gamma function, defined by: $$\Gamma(t)=\int_{0}^{\infty}x^{t-1}\exp(-x)\textrm{d}x.$$
The probability density function (\textit{PDF}) is given by:
\begin{equation}\label{eq:pdf_chi}
\displaystyle f_{\chi^2}\left(\psi \right) = \frac{1}{2^{n/2} \Gamma\left(n/2\right)} \psi^{n/2-1}\exp\left(-\frac{\psi}{2} \right).
\end{equation}

\item The generalized case of a globally minimal quadratic basin, where the Hessian matrix is positive definite with the following eigendecomposition form,
$$\mathcal{H}=\mathcal{U}\mathcal{D}\mathcal{U}^{-1}~~~~~\mathcal{D}=\textrm{diag}\left[\Delta_1,\ldots,\Delta_n \right],$$
with $\left\{\Delta_i \right\}_{i=1}^{n}$ being the eigenvalues.
The random variable $\psi= J(\vec{z})$ now obeys a generalized $\chi^2$-distribution, whose \textit{exact distribution function} is described as follows \cite{FD68_NASA}:
\begin{equation}
\begin{array}{r}
\medskip
\displaystyle F_{\mathcal{H}\chi^2}(\psi) = \int_0^{\infty} \frac{2}{\pi} \frac{\sin\frac{t\psi}{2}}{t} \cos \left(-t\psi + \frac{1}{2} \sum_{j=1}^{n}\tan^{-1}2\Delta_j t \right) \\
\displaystyle \times \prod_{j=1}^{n}\left(1+\Delta_j^2 t^2 \right)^{-\frac{1}{4}} ~ \textrm{d}t,
\end{array}
\end{equation}
with an unknown closed form. At the same time, this CDF is known to follow an \textit{approximation} \cite{FD68_NASA}, 
\begin{equation}\label{eq:Nchi}
\displaystyle F_{\mathcal{T}\chi^2}\left(\psi \right) = \frac{\Upsilon^{\eta}}{\Gamma\left(\eta\right)} \int_0^{\psi}t^{\eta-1}\exp\left(-\Upsilon t \right) ~ \textrm{d}t,
\end{equation}
with $\Upsilon$ and $\eta$ accounting for matching the first two moments of $\vec{z}^{T}\mathcal{H}\vec{z}$ (and the subscript $\mathcal{T}$ marks the transformed distribution):
\begin{equation}
\Upsilon=\frac{1}{2}\frac{\sum_{i=1}^{n}\Delta_i}{\sum_{i=1}^{n}\Delta_i^2},~~~\eta=\frac{1}{2}\frac{\left(\sum_{i=1}^{n}\Delta_i\right)^2}{\sum_{i=1}^{n}\Delta_i^2}.
\end{equation}
The density function of this approximation reads:
\begin{equation}\label{eq:Nchi_density}
\displaystyle f_{\mathcal{T}\chi^2}\left(\psi \right) = \frac{\Upsilon^{\eta}}{\Gamma\left(\eta\right)} \psi^{\eta-1}\exp\left(-\Upsilon \psi \right).
\end{equation}
\textbf{The accuracy of this approximation depends upon the standard deviation of the eigenvalues \cite{FD68_NASA}, which is clearly related to the so-called condition number}. We assume that standard deviation to be moderate, and we thus adopt this approximation herein.
For the isotropic case, it can be easily verified that Eq.\ \ref{eq:Nchi_density} reduces to Eq.\ \ref{eq:pdf_chi}.
\end{compactenum}

\medskip

We conclude this section, by \textbf{summarizing the relevant notation}:
The random vector $\vec{z}$ is a normal Gaussian mutation and $\psi=J(\vec{z})$.
The random vectors $\vec{x}_1,\ldots,\vec{x}_\lambda$ are $\lambda$ independent copies of $\vec{z}$, and $J_i=J(\vec{x}_i)$.
The winner is $\vec{y}$, and $\omega = J(\vec{y})$. 
The matrix $\mathcal{H}$ is the Hessian about the optimum $\vec{x}^{*}$, and $\mathcal{C}$ is the covariance matrix of $\vec{y}$.

\section{Covariance Matrix Formulation} \label{sec:covariance}
In the current section we formulate the covariance matrix by means of its defining density functions and then prove it commutes with the landscape Hessian.

By construction, the origin is set at the parent search-point, which is located at the optimum.
Analytically, the covariance elements are thus reduced to the following \emph{expectation values}:  
\begin{equation}
\label{eq:Cov0}
\boxed{
\displaystyle \mathcal{C}_{ij} = \int x_i x_j \texttt{PDF}_{\vec{y}}\left( \vec{x}\right) \textrm{d}\vec{x}},
\end{equation}
where $\texttt{PDF}_{\vec{y}}\left( \vec{x}\right)$ is an $n$-dimensional density function characterizing the \textit{winning} decision variables about the optimum. 
{\bf In essence, the current study aims at understanding this expression in Eq.\ \ref{eq:Cov0}. 
To this end, revealing the nature of $\texttt{PDF}_{\vec{y}}$ is necessary for the interpretation of the covariance matrix}.
Importantly, the selection mechanism of the heuristic is blind to the location of the candidate solutions in the search space, and its sole criterion is the ranked function values. 
Specifically, in the \textit{decision-space perspective}, the density function of a \textit{winning} vector of decision variables $\vec{y}$ is related to the density of the \textit{winning} function value $\omega$ via the following relation:
\begin{equation}\label{eq:x_pdf_2}
\boxed{
\displaystyle \texttt{PDF}_{\vec{y}}\left(\vec{x}\right) =   \texttt{PDF}_{\omega}\left( J\left(\vec{x} \right) \right) \cdot \frac{\texttt{PDF}_{\vec{z}}\left(\vec{x}\right)}{\texttt{PDF}_{\psi}\left( J\left(\vec{x} \right)\right)} }
\end{equation}
with $\texttt{PDF}_{\vec{z}}$ denoting the density function for generating an individual, and $\texttt{PDF}_{\psi}$ denoting the density function of the objective function values (Eqs.\ \ref{eq:pdf_chi} and \ref{eq:Nchi_density}).
A brief justification follows. 
The density functions satisfy the conditional probability relation:
\begin{equation}
\texttt{PDF}_{\vec{y}}\left(\vec{x}\right) = \texttt{PDF}_{\omega}\left( J\left(\vec{x} \right) \right) \cdot \texttt{PDF}_{\vec{y}\mid \omega}\left( \vec{x} \mid J\left(\vec{x} \right) \right).
\end{equation}
Now consider the distribution of $\left[ \vec{y} ; \omega \right]$ on $\mathbb{R}^{n+1}$.
The density of $\vec{y}$ conditioned on the value of $J\left( \vec{y} \right)$ is that of a normal Gaussian subject to this conditioning, since we may sample $\left[ \vec{y} ; \omega \right]$ by the following construction: 
First sample $\left\{ J_{1},\ldots ,J_{\lambda} \right\}$ according to $\texttt{PDF}_{\psi}$ independently.
Then sample $\left\{ \vec{x}_{1},\ldots,\vec{x}_{\lambda} \right\}$ conditioned on the values of $J_{1},\ldots,J_{\lambda}$ independently.
Finally, $J$ may be set to $J_{\ell}=\omega$ that is minimal and $\vec{y}$ set to the respective $\vec{x}_{\ell}$. 
In other words, following selection, a winning value of $J$ is chosen to be $\omega$, and the corresponding $\vec{x}$ becomes the winning vector $\vec{y}$. 
Importantly, the winning vector $\vec{y}$ conditioned upon the winning value $\omega$ is generated in the same manner as a normally-distributed $\vec{z}$ conditioned upon $\psi$. 
As a result, the conditional probability for the generation of $\vec{y}$ conditioned upon $\omega$ is the same as that for the creation of $\vec{z}$ conditioned upon $\psi$, i.e., $\texttt{PDF}_{\vec{y}\mid \omega}=\texttt{PDF}_{\vec{z}\mid \psi}$.
This density therefore reads:
\begin{equation}
\texttt{PDF}_{\vec{y}\mid \omega}\left( \vec{x} \mid J\left(\vec{x} \right) \right) = 
\texttt{PDF}_{\vec{z}\mid \psi}\left( \vec{x} \mid J\left(\vec{x} \right) \right) = 
 \frac{\texttt{PDF}_{\vec{z}}\left(\vec{x}\right)}{\texttt{PDF}_{\psi}\left( J\left(\vec{x} \right)\right)}.
\end{equation}

\begin{thm} \label{thm:commuting}
The covariance matrix and the Hessian are \textbf{commuting matrices} when the objective function follows the quadratic approximation.
\end{thm}
\begin{proof}
Given the density function in Eq.\ \ref{eq:x_pdf_2}, the objective function is assumed to satisfy $J\left( \vec{x}\right) = \vec{x}^T \cdot \mathcal{H} \cdot \vec{x}$, and the covariance matrix reads:
\begin{equation}
\label{eq:Cov1}
\displaystyle \mathcal{C}_{ij} = \int x_i x_j \texttt{PDF}_{\omega}\left( \vec{x}^T \cdot \mathcal{H} \cdot \vec{x} \right) \cdot \frac{\texttt{PDF}_{\vec{z}}\left(\vec{x}\right)}{\texttt{PDF}_{\psi}\left( \vec{x}^T \cdot \mathcal{H} \cdot \vec{x} \right)} \textrm{d}\vec{x}.
\end{equation}
Consider the orthogonal matrix $\mathcal{U}$, which diagonalizes $\mathcal{H}$ into $\mathcal{D}$ and possesses a determinant of value $1$:  
\begin{equation*}
\begin{array}{l}
\medskip
\displaystyle \mathcal{U}^{-1} \mathcal{H} \mathcal{U} = \mathcal{D} \equiv \textrm{diag}\left[\Delta_1,\Delta_2,\ldots,\Delta_n\right]\\
\medskip
\displaystyle \vec{\vartheta} = \mathcal{U}^{-1} \vec{x}\\
\displaystyle \textrm{d}\vec{\vartheta} = \textrm{d}\vec{x}
\end{array}.
\end{equation*}
We target the integral $\mathcal{I}_{ij}=\left( \mathcal{U}^{-1} \mathcal{C} \mathcal{U}\right)_{ij}$ and apply a change of variables into $\vec{\vartheta}$ (after changing order of summations):
\begin{equation}
\begin{array}{r}
\medskip
\displaystyle \mathcal{I}_{ij} = \frac{1}{\sqrt{\left(2\pi \right)^n}} 
\int_{-\infty}^{+\infty}\int_{-\infty}^{+\infty}\cdots \int_{-\infty}^{+\infty} 
\vartheta_i \vartheta_j  \exp \left(-\frac{1}{2} \vec{\vartheta}^{T}\vec{\vartheta} \right) \times \\
\displaystyle \times \frac{\texttt{PDF}_{\omega}\left( \vec{\vartheta}^T \cdot \mathcal{D} \cdot \vec{\vartheta} \right)}{\texttt{PDF}_{\psi}\left( \vec{\vartheta}^T \cdot \mathcal{D} \cdot \vec{\vartheta} \right)}
\textrm{d}\vartheta_1\textrm{d}\vartheta_2\cdots \textrm{d}\vartheta_n.
\end{array}
\end{equation}
$\mathcal{I}_{ij}$ vanishes for any $i\neq j$ due to symmetry considerations: 
the overall integrand is an \textit{odd} function, because all the terms are \textit{even} functions, except for $\vartheta_j$, $\vartheta_i$ when they differ. Therefore, the integration over the entire domain yields zero.
Hence, $\mathcal{I}$ is the diagonalized form of $\mathcal{C}$, with $\mathcal{U}$ holding the eigenvectors.
$\mathcal{C}$ is thus diagonalized by the same eigenvectors as $\mathcal{H}$, and therefore, by definition, they are commuting matrices, as claimed.
\end{proof}

\section{Analytic Approximation} \label{sec:solving_singlewinner}
In this section we provide an approximation for $\texttt{PDF}_{\omega}\left(J\left(\vec{x}\right)\right)$ and consequently for $\texttt{PDF}_{\vec{y}}\left(\vec{x}\right)$ in order to explicitly calculate the covariance matrix using Eq.\ \ref{eq:Cov0}.

A non-elitist multi-child selection is considered here, where in each iteration a single individual is deterministically selected out of $\lambda$ generated offspring. In particular, consider a random sample from an absolutely continuous population with \textit{density} $\texttt{PDF}_{\psi}\left(\psi\right)$ and \textit{distribution} $\texttt{CDF}_{\psi}\left(\psi\right)$.
In order to formulate the \textit{density} of those \textit{winners}, it is convenient to first characterize the distribution function of the winning event amongst $\lambda$ candidates\footnote{Gupta \cite{Gupta1960} showed that when the dimension $n$ is \textit{even}, the distribution of the winners for the $\chi^2$ distribution (isotropic basin case) possesses a simple form:
\begin{equation}\label{eq:y_cdf_even}
\displaystyle \texttt{CDF}_{\omega}^{(n=2m)}\left(\psi\right) =   1 - \exp\left(-\lambda\frac{\psi}{2} \right) \left(\sum_{\jmath=0}^{\frac{n}{2}-1}\frac{\psi^{\jmath}}{\jmath!} \right)^{\lambda}
\end{equation}.}:
\begin{equation}\label{eq:y_cdf}
\displaystyle \texttt{CDF}_{\omega}\left(\psi\right) =  \texttt{Pr}\left\{ \omega \leq \psi \right\} = 1 - \left( 1-\texttt{CDF}_{\psi}\left( \psi \right)\right)^{\lambda}.
\end{equation}
The density function is obtained upon differentiating Eq.\ \ref{eq:y_cdf}:
\begin{equation}\label{eq:y_pdf}
\displaystyle \texttt{PDF}_{\omega}\left( \psi \right) =   \lambda \cdot \left( 1-\texttt{CDF}_{\psi}\left(\psi \right)\right)^{\lambda-1}\cdot \texttt{PDF}_{\psi}\left(\psi\right).
\end{equation}
Upon substituting the explicit forms into $\texttt{CDF}_{\psi}$ and $\texttt{PDF}_{\psi}$ (using either Eqs.\ (\ref{eq:cdf_chi},\ref{eq:pdf_chi}) for the \textit{standard} $\chi^2$ or Eqs.\ (\ref{eq:Nchi},\ref{eq:Nchi_density}) for the \textit{generalized} $\chi^2$), the desired density function $\texttt{PDF}_{\omega}\left(J\left(\vec{x}\right)\right)$ is obtained, however not in a closed form.\\

Gupta \cite{Gupta1960} derived explicit order statistic results from the Gamma distribution, to which the $\chi^2$ distribution belongs, including the distribution function as well as moments of the $k^{th}$ order statistic. 
Such results could reveal a closed form for $\texttt{PDF}_{\omega}\left(J\left(\vec{x}\right)\right)$, which seems cumbersome and far too complex to address when targeting Eq.\ \ref{eq:Cov0}.
Next, we will seek an \textit{approximation} for $\texttt{PDF}_{\omega}\left(J\left(\vec{x}\right)\right)$, which will enable us to realize the relation of Eq.\ \ref{eq:x_pdf_2} when large values of $\lambda$ are assumed.

\subsection{Limit Distributions of Order Statistics}
We treat the derived winners' distribution for large sample sizes, i.e., when the population size $\lambda$ tends to infinity.
We denote the CDF in Eq.\ \ref{eq:y_cdf} with a subscript $\lambda$, 
$\mathcal{L}_{\lambda}\left(\psi\right) = 1 - \left( 1-\texttt{CDF}_{\psi}\left(\psi\right)\right)^{\lambda}$,
and consider the limit when $\lambda$ tends to infinity:
\begin{equation*}
\displaystyle \lim_{\lambda\longrightarrow \infty} \mathcal{L}_{\lambda}\left(\psi\right) = 
\left\{\begin{array}{l}
0~~~~~\textrm{if}~\texttt{CDF}_{\psi}\left(\psi\right)=0\\
1~~~~~\textrm{if}~\texttt{CDF}_{\psi}\left(\psi\right)>0
\end{array} \right.
\end{equation*}
According to the Fisher-Tippett theorem \cite{FisherTippett1928}, also known as the \textit{extremal types theorem}, the \textit{von-Mises} family of distributions for \textbf{minima} (or the minimal generalized extreme value distributions ($\textrm{GEVD}_{\min}$)) are the only non-degenerate family of distributions satisfying this limit. 
They are characterized as a unified family of distributions by the following CDF:
\begin{equation} \label{eq:minimaGEVD}
\mathcal{L}_{\kappa}\left( \psi;\kappa_1,\kappa_2,\kappa_3 \right) = 1-\exp \left\{- \left[1+ \kappa_3\left(\frac{\psi-\kappa_1}{\kappa_2}\right)\right]^{1/\kappa_3} \right\}.
\end{equation}
Furthermore, since the minimum distribution moves toward the origin as $\lambda$ increases, normalizing constants are needed to avoid degeneracy and to obtain:
\begin{equation}\label{eq:limitMin}
\begin{array}{l}
\displaystyle \lim_{\lambda\longrightarrow \infty} \mathcal{L}_{\lambda}\left(a^{*}_{\lambda}\psi+b^{*}_{\lambda}\right) =\\ 
\displaystyle \lim_{\lambda\longrightarrow \infty} 1 - \left( 1-\texttt{CDF}_{\psi}\left(a^{*}_{\lambda}\psi+b^{*}_{\lambda}\right)\right)^{\lambda}=\mathcal{L}\left(\psi\right)~~~\forall \psi
\end{array}
\end{equation}
The location parameter, $\kappa_1$, and the scale parameter, $\kappa_2$, are obviously interlinked to the aforementioned normalizing constants. The shape parameter, $\kappa_3$, determines the identity of the characteristic CDF, namely either Weibull, Gumbell, or Frech{\'e}t. This parameter is evaluated by means of the following limit, whose existence is a necessary and sufficient condition for a continuous distribution function $\texttt{CDF}_{\psi}\left(\psi\right)$ to belong to the domain of attraction for minima of $\mathcal{L}_{\kappa}\left( \psi\right)$ :
\begin{equation}
\label{eq:kappa_def}
\displaystyle \kappa_3 = \lim_{\varepsilon\longrightarrow 0} -\log_{2} \frac{\texttt{CDF}_{\psi}^{-1}\left(\varepsilon \right)-\texttt{CDF}_{\psi}^{-1}\left(2\varepsilon \right)}{\texttt{CDF}_{\psi}^{-1}\left(2\varepsilon \right)-\texttt{CDF}_{\psi}^{-1}\left(4\varepsilon \right)},
\end{equation}
where $\texttt{CDF}_{\psi}^{-1}$ refers to the inverse CDF (the quantile function of $\texttt{CDF}_{\psi}\left(\psi\right)$; see Theorem 9.6 in \cite{Castillo2004} [pp. 204-205]):
\begin{itemize}
\item If $\kappa_3 > 0$, $\texttt{CDF}_{\psi}\left(\psi\right)$ belongs to the Weibull minimal domain of attraction,
\item if $\kappa_3 = 0$, $\texttt{CDF}_{\psi}\left(\psi\right)$ belongs to the Gumbel minimal domain of attraction, and
\item if $\kappa_3 < 0$, $\texttt{CDF}_{\psi}\left(\psi\right)$ belongs to the Frech{\'e}t minimal domain of attraction.
\end{itemize}

Note that Rudolph had already taken a related mathematical approach, which he termed \textit{asymptotic theory of extreme order statistics}, to characterize convergence properties of ESs on a class of convex objective functions \cite{Rudolph97}.
Also, GEVD is introduced to the broad perspective of Stochastic Global Optimization in \cite{zhigljavsky2007stochastic}, a book which also constitutes a proper mathematical reference for this topic, yet in a slightly different light. 

\begin{lm}
For the standard $\chi^2$ distribution (i.e., isotropic basin), the limit for Eq.\ \ref{eq:kappa_def} exists and reads $\kappa_3=2/n$.
\end{lm}
\begin{proof}
The limit needs to be evaluated about $\varepsilon\longrightarrow 0$. 
By inserting an asymptotic expansion of the Gamma function's integrand, the overall CDF $F_{\chi^2}$ (Eq. \ref{eq:cdf_chi}) may be written and approximated using Stirling's formula as 
\begin{equation*}
\begin{array}{l}
\displaystyle F_{\chi^2}\left(\varepsilon \right) = \frac{1}{2^{n/2}\Gamma\left(n/2\right)}\varepsilon^{\frac{n}{2}}\sum_{k=0}^{\infty}\frac{\left(-1 \right)^k \left(\frac{\varepsilon}{2}\right)^{k}}{\left( \frac{n}{2}+k\right) k!} \\
\displaystyle \approx  \frac{2}{n\cdot 2^{n/2}\Gamma\left(n/2\right)} \varepsilon^{\frac{n}{2}} \approx \left(\frac{\varepsilon}{4}\frac{e}{n} \right)^{\frac{n}{2}}
\end{array},
\end{equation*}
taking only the zeroth-order term in the sum into consideration. The quantile (inverse) function has the form: 
\begin{equation}
\displaystyle F_{\chi^2}^{-1}\left(\varepsilon \right) \approx \frac{4n}{e}\cdot \varepsilon^{\frac{2}{n}}.
\end{equation}
Targeting the limit in Eq.\ \ref{eq:kappa_def} yields
\begin{equation}
\displaystyle \frac{F_{\chi^2}^{-1}\left((\varepsilon \right)-F_{\chi^2}^{-1}\left(2\varepsilon \right)}{F_{\chi^2}^{-1}\left(2\varepsilon \right)-F_{\chi^2}^{-1}\left(4\varepsilon \right)} \approx \frac{\varepsilon^{\frac{2}{n}}\left(2^{\frac{2}{n}} -1 \right)}{\varepsilon^{\frac{2}{n}}\left(4^{\frac{2}{n}} - 2^{\frac{2}{n}} \right)} = \frac{1}{2^{\frac{2}{n}}},
\end{equation}
which allows to conclude with:
\begin{equation}
\label{eq:limit_kappa}
\displaystyle \kappa_3 = \lim_{\varepsilon\longrightarrow 0} -\log_{2} \frac{F_{\chi^2}^{-1}\left(\varepsilon \right)-F_{\chi^2}^{-1}\left(2\varepsilon \right)}{F_{\chi^2}^{-1}\left(2\varepsilon \right)-F_{\chi^2}^{-1}\left(4\varepsilon \right)} = \frac{2}{n}.
\end{equation}
 \end{proof}
\begin{lm}
For the generalized $\chi^2$ distribution (i.e., non-isotropic basin), the limit for Eq.\ \ref{eq:kappa_def} exists and also reads $\kappa_3=2/n$.
\end{lm}
\begin{proof}
In the limit $\varepsilon\longrightarrow 0$, the generalized distribution (Eq.\ \ref{eq:Nchi}) has a similar CDF, with an approximated quantile function
\begin{equation}
\displaystyle F_{\mathcal{T}\chi^2}^{-1}\left(\varepsilon \right) \approx \frac{4n}{e}\cdot \varepsilon^{\frac{2}{n}},
\end{equation}
and thus Eq.\ \ref{eq:limit_kappa} holds as is.
\end{proof}
\begin{lm}
For the standard and generalized $\chi^2$ distributions, the normalizing constants 
\begin{equation*}
a^{*}_{\lambda}=F_{\chi^2}^{-1}\left(\frac{1}{\lambda} \right),~~~b^{*}_{\lambda}=\mathrm{inf} \left\{\psi \left|F_{\chi^2}\left( \psi \right)>0\right.\right\}=0 
\end{equation*}
ensure that the limit distribution of Eq.\ \ref{eq:limitMin} is not degenerate.
\end{lm}
\begin{proof}
Given the constants
$$a^{*}_{\lambda}=F_{\chi^2}^{-1}\left(\frac{1}{\lambda} \right)\approx \frac{4n}{e}\left(\frac{1}{\lambda} \right)^{2/n},~~~b^{*}_{\lambda}=0,$$
the limit becomes (using $F_{\chi^2}(\varepsilon)=F_{\mathcal{T}\chi^2}(\varepsilon) = \left(\frac{\varepsilon}{4}\frac{e}{n} \right)^{\frac{n}{2}}$)
\begin{equation}
\begin{array}{l}
\medskip
\displaystyle \lim_{\lambda\longrightarrow \infty} 1 - \left\{ 1-\texttt{CDF}_{\psi}\left[\psi \frac{4n}{e}\left(\frac{1}{\lambda} \right)^{2/n} \right]\right\}^{\lambda} = \\
\medskip
\displaystyle \lim_{\lambda\longrightarrow \infty} 1 - \left[1-r\left(n\right) \left(\frac{\psi^{n/2}}{\lambda} \right) \right]^{\lambda} =\\
\displaystyle 1-\exp\left[-r\left(n\right) \left(\psi \right)^{n/2} \right]
\end{array},
\end{equation}
with $r\left(n\right)=\left(\frac{e}{4n} \right)^{n/2}$. 
Hence, the limit distribution exists and is not degenerate, as claimed.
\end{proof}
\begin{crl}
Since the shape parameter $\kappa_3$ is always positive, the extreme minima of the $\chi^2$-distributions belong to the Weibull domain of attraction. 
The normalized extreme minima, $(\omega-b^{*}_{\lambda})/a^{*}_{\lambda}$, may be represented by a random variable $\tilde{\psi}$, which then
reduces Eq.\ \ref{eq:minimaGEVD} to the following transformed CDF (importantly, the so-called \textbf{tail index} reads $1/\kappa_3 = n/2$):
\end{crl}
\begin{equation}\label{eq:weibull}
\displaystyle \texttt{CDF}_{\omega}\left(\psi\right)\xrightarrow{\lambda\to\infty} \mathcal{W}\left( \tilde{\psi} \right) = 1-\exp \left(-\tilde{\psi}^{n/2} \right)
\end{equation}

See \cite{Castillo2004} and \cite{Embrechts1997} for an overview on the family of generalized extreme value distributions and on the limit distributions of order statistics.
In particular, see table 9.1 in \cite{Castillo2004}[p. 200] for the relationship between the parameters of the GEVD and the Weibull distribution, which allows the reduction of Eq.\ \ref{eq:minimaGEVD} to Eq.\ \ref{eq:weibull}.
Also, for the exact determination of the tail index value, when assuming certain conditions on the sampling distribution, see Theorem 2.3 in \cite{zhigljavsky2007stochastic}.

\begin{crl}
Under the GEVD approximation for treating large populations, $\lambda\rightarrow \infty$, upon normalizing the variable to $\tilde{\psi}=\left(\omega-b^{*}_{\lambda}\right)/a^{*}_{\lambda}$ and using the tail index result $1/\kappa_3=\frac{n}{2}$, the CDF and PDF forms for the single winning event read:
\end{crl}
\begin{equation}\label{eq:cdf_y_evd}
\begin{array}{l}
\medskip
\displaystyle \texttt{CDF}_{\omega}^{\textrm{GEVD}}\left( \tilde{\psi} \right) = 1 -\exp \left(-\tilde{\psi}^{\frac{n}{2}} \right)\\
\displaystyle \texttt{PDF}_{\omega}^{\textrm{GEVD}}\left( \tilde{\psi} \right) =  \frac{n}{2} \tilde{\psi}^{\frac{n}{2}-1} \exp \left(-\tilde{\psi}^{\frac{n}{2}} \right)
\end{array}
\end{equation}

\subsection{Covariance Derivation}
By setting the \textit{Weibull} form as the characteristic density $\texttt{PDF}_{\omega}$, we may rewrite Eq.\ \ref{eq:Cov0} by utilizing Eq.\ \ref{eq:x_pdf_2} as follows with the normalized $\tilde{J}(\vec{x}) \equiv \left( J(\vec{x})-b^{*}_{\lambda} \right) /a^{*}_{\lambda}$:
\begin{equation}
\begin{array}{l}
\medskip
\displaystyle \mathcal{C}_{ij} = \int_{-\infty}^{+\infty}\cdots \int_{-\infty}^{+\infty} 
x_i x_j \frac{n}{2} \tilde{J}(\vec{x})^{\frac{n}{2}-1} \exp\left[-\tilde{J}(\vec{x})^{\frac{n}{2}}\right] \times \\
\displaystyle  \times \frac{\frac{1}{\sqrt{\left(2\pi \right)^n}}\exp\left(-\frac{1}{2} \vec{x}^T \vec{x} \right)}{\frac{\Upsilon^{\eta}}{\Gamma\left(\eta\right)} J(\vec{x})^{\eta-1}\exp\left(-\Upsilon J(\vec{x}) \right)} \textrm{d}x_1\textrm{d}x_2\cdots \textrm{d}x_n.
\end{array}
\end{equation}
$J$ is assumed here to satisfy $J\left( \vec{x}\right) = \vec{x}^T \cdot \mathcal{H} \cdot \vec{x}$,
and must be normalized only for the $\texttt{PDF}_{\omega}$ term by means of $a^{*}_{\lambda}$ alone since $b^{*}_{\lambda}=0$:
\begin{equation} \label{eq:Cij}
\boxed{
\begin{array}{l}
\medskip
\displaystyle \mathcal{C}_{ij} = \int_{-\infty}^{+\infty}\cdots \int_{-\infty}^{+\infty} 
N_{\mathcal{C}}  x_i x_j \left(\vec{x}^T \mathcal{H} \vec{x} \right)^{\frac{n}{2}-\eta} \times \\
\displaystyle \times \exp\left[\Upsilon\vec{x}^T \mathcal{H} \vec{x} -\left(\frac{\vec{x}^T \mathcal{H} \vec{x}}{a^{*}_{\lambda}} \right)^{\frac{n}{2}} - \frac{1}{2}\vec{x}^T \vec{x}\right] \textrm{d}x_1\textrm{d}x_2\cdots \textrm{d}x_n.
\end{array} }
\end{equation}
with a normalizing constant $N_{\mathcal{C}}=\frac{n\Gamma(\eta)}{2\Upsilon^{\eta} \left(a^{*}_{\lambda}\right)^{\frac{n}{2}-1}\sqrt{\left(2\pi \right)^n}}$.

For the isotropic case, $\mathcal{H}=h_0 \mathbf{I}$, the integration is straightforward ($\eta=\frac{n}{2}$, $\Upsilon=\frac{1}{2h_0}$) -- the attained covariance is proportional to the inverse Hessian, multiplied by an explicit factor:
\begin{equation}\label{eq:isoResult}
\displaystyle \mathcal{C}^{(\mathcal{H}=h_0\mathbf{I})} =  \frac{\Gamma(\frac{n}{2})\cdot \Gamma\left(1+\frac{2}{n} \right) \cdot c\left( n \right) \cdot a^{*}_{\lambda}}{2\pi^{n/2}}  \cdot \mathcal{H}^{-1}
\end{equation}
wherein
\begin{equation}
\displaystyle c\left( n \right)=\left\{
\begin{array}{ll} \frac{\pi^m}{m!} & n=2m \\ \frac{2^{m+1}\pi^m}{1 \cdot 3\cdot 5 \cdots \left(2m+1\right)} & n=2m+1
\end{array}\right.
\end{equation}

For the general case of any positive-definite Hessian $\mathcal{H}$, the integral in Eq.\ \ref{eq:Cij} has an unknown closed form.
We were able, nevertheless, to numerically corroborate it for $n=3$.
We note that this form of the covariance also commutes with the Hessian, in line with the theorem discussed above.

\section{Simulation Study}\label{sec:simulation}
We implemented our model into a numerical procedure in order to compare the statistical measures to the analytical calculations in practice, adhering to Algorithm 1, $\mathcal{C}^{\texttt{stat}}$, 
statistically sampled as described therein.
Numerical validation is provided here for two aspects of our theoretical work: Theorem \ref{thm:commuting} and the analytic approximation for the covariance matrix.

\subsection{$\mathcal{C}^{\texttt{stat}}$ \textit{versus} $\mathcal{H}$}
We generated a large set of random positive-definite matrices at various dimensions $\left\{n_{\jmath} \right\}$ with a spectrum of condition numbers.
For each trial $\jmath$, the numerical procedure generated a random symmetric matrix $A_{\jmath}$, diagonalized it into a set of orthonormal eigenvectors $U_{\jmath}$, drew $n_{\jmath}$ random positive numbers in a diagonal matrix $D_{\jmath}$, and set $\mathcal{H}_{\jmath}=U_{\jmath} A_{\jmath} U_{\jmath}^{-1}$. 
We then applied Algorithm 1 by considering $\left\{\mathcal{H}_{\jmath} \right\}$ as landscape Hessians. 

The resultant covariance matrices, $\left\{\ \mathcal{C}^{\texttt{stat}}_{\jmath} \right\}$, were diagonalized and compared to the Hessian matrices and their eigendecomposition -- which always matched.
In practice, it was evident that the two matrices always commute (applying the commutator operator yields a zero matrix to a practical precision considering the \textit{max norm}), 
\begin{equation}
\displaystyle \forall \jmath ~~~ \| \mathcal{H}_{\jmath} \mathcal{C}^{\texttt{stat}}_{\jmath} - \mathcal{C}^{\texttt{stat}}_{\jmath} \mathcal{H}_{\jmath}\|_{\max} < 10^{-1},
\end{equation}
as claimed. However, the covariance matrices were \textbf{not} the inverse forms of the Hessian matrices.

\subsection{Analytic Approximation for $\mathcal{C}^{\texttt{stat}}$}
Here, we corroborated the analytic approximation for the covariance matrix.
To this end, we considered four quadratic basins of attraction at various search-space dimensions:
\begin{compactenum}[(H-1)]
\item $n=3$, $\mathcal{H}_1=\left[\sqrt{2}/2~ 0.25~ 0.1;~ 0.25~ 1~ 0;~ 0.1~ 0~ \sqrt{2} \right]$
\item $n=10$, $\mathcal{H}_2=\textrm{diag}\left[1.0,1.5,\ldots,5.5 \right]$
\item $n=30$, $\mathcal{H}_3=\textrm{diag}\left[\vec{I}^{10},2\cdot\vec{I}^{10},3\cdot\vec{I}^{10} \right]$
\item $n=100$, $\mathcal{H}_4= 2.0 \cdot \mathbf{I}^{100 \times 100}$
\end{compactenum}
\subsubsection{Validating the Approximated $\chi^2$ Density $f_{\mathcal{T}\chi^2}$}
Figure \ref{table:approxChi2} depicts the approximated density functions of the generalized $\chi^2$ distribution, $f_{\mathcal{T}\chi^2}$ (Eq.\ \ref{eq:Nchi_density}) for the four Hessian forms (H-1)-(H-4), which evidently constitute sound approximations.

\subsubsection{Validating the Winners' Densities $\texttt{PDF}_{\omega}$ and $\texttt{PDF}_{\omega}^{\textrm{GEVD}}$}
Figures \ref{table:winners} and \ref{table:H2winners} provide validation for the winners' density, which was exactly described by $\texttt{PDF}_{\omega}$ in Eq.\ \ref{eq:y_pdf}, and was later approximated by $\texttt{PDF}_{\omega}^{\textrm{GEVD}}$ in Eq.\ \ref{eq:cdf_y_evd} for large $\lambda$.
Interestingly, $\texttt{PDF}_{\omega}$, which is realized here by the approximated generalized $\chi^2$ distribution $F_{\mathcal{T}\chi^2}$, exhibits decreased accuracy on $\mathcal{H}_1$, $\mathcal{H}_2$ and $\mathcal{H}_3$. Evidently, it is highly sensitive to the approximation error of $F_{\mathcal{T}\chi^2}$, which is amplified by the exponent $\lambda$. 
At the same time,  $\texttt{PDF}_{\omega}^{\textrm{GEVD}}$ exhibits decreased accuracy on $\mathcal{H}_2$, $\mathcal{H}_3$ and $\mathcal{H}_4$, due to its sensitivity to the population size $\lambda$. 
Indeed, improvements for this approximation were evident when $\lambda$ was increased (see additional settings on Figure \ref{table:H2winners}).

\subsubsection{Validating the Approximated Integral}
Finally, we compared $\mathcal{C}^{\texttt{stat}}$ 
to the obtained analytical approximations.
For the isotropic case, the result of Eq.\ \ref{eq:isoResult} has been successfully corroborated for a range of search-space dimensions $n$. For instance, $\mathcal{C}^{\texttt{stat}}$ for the $100$-dimensional case (H-4) was constructed with $\lambda=5000$ and over $500,000$ iterations to obtain a diagonal with an expected value $0.5617 \pm 0.0012$. Eq.\ \ref{eq:isoResult} obtained a value of $0.5680$.

For the general case, we considered the $3$-dimensional case (H-1). $\mathcal{C}^{\texttt{stat}}$ was constructed with $\lambda=20$ and over $10,000$ iterations, to be presented side-by-side with the numerical integration of Eq.\ \ref{eq:Cij} in Table \ref{table:H1}. Additionally, their explicit eigenvectors are provided therein.
\begin{figure}
\centering Simulation Study: The distribution of $J(\vec{x})$ over a single iteration\\
\begin{tabular}{l | l}
\hline
$\mathcal{H}_1~(n=3):$ & $\mathcal{H}_2~(n=10):$ \\
\hline
\includegraphics[width=0.45\columnwidth]{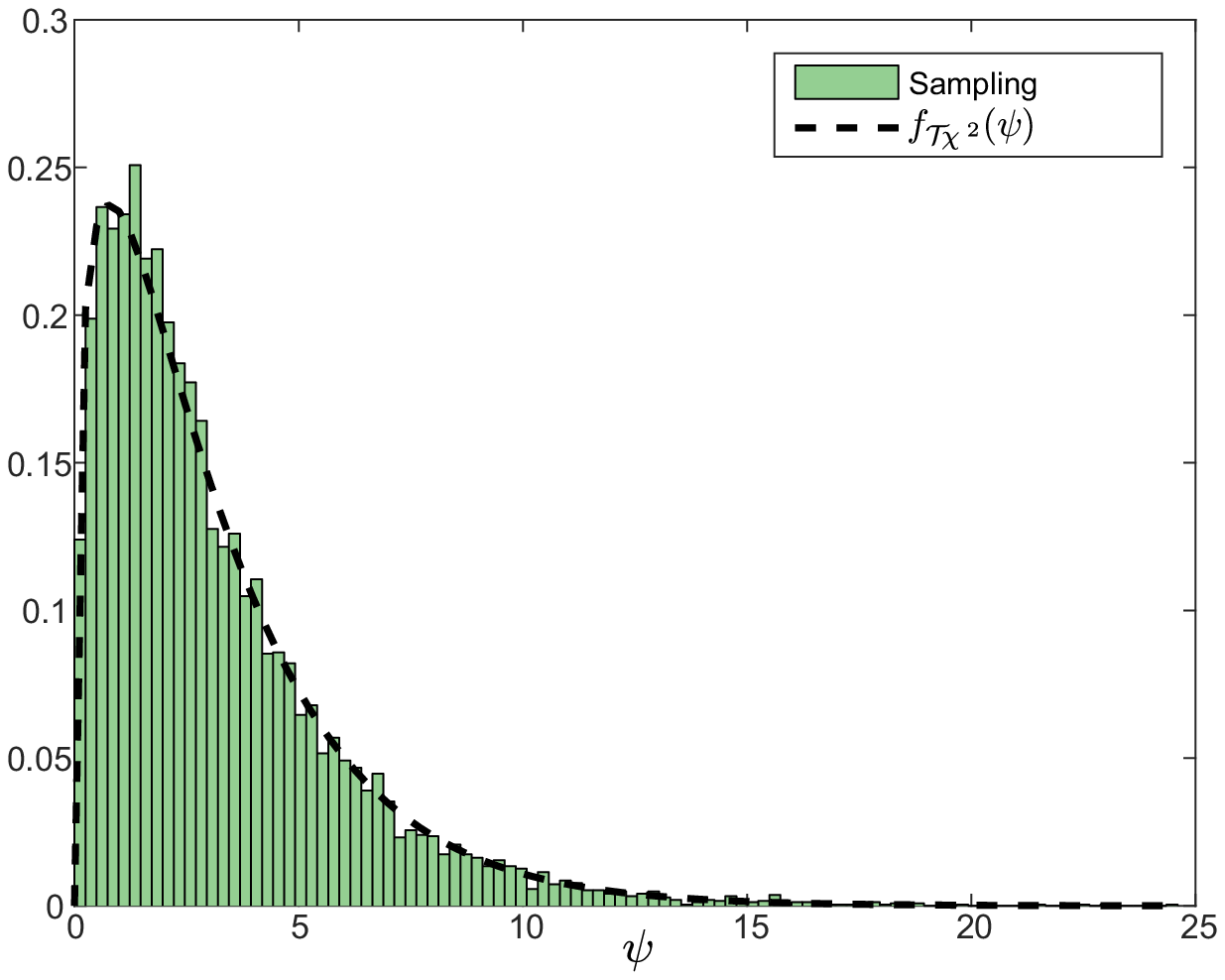} & \includegraphics[width=0.45\columnwidth]{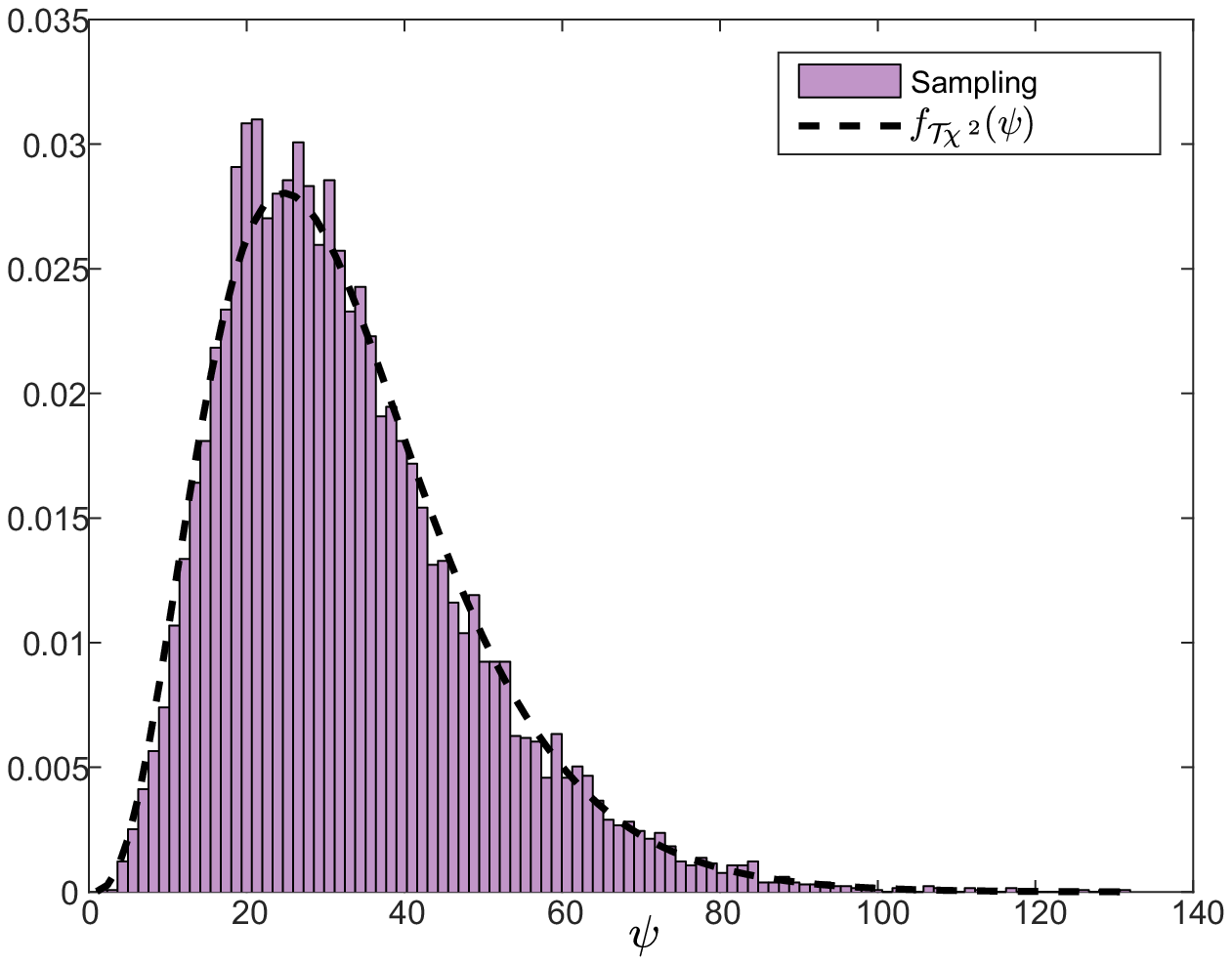}\\
\hline
\hline
$\mathcal{H}_3~(n=30):$ & $\mathcal{H}_4~(n=100):$\\
\hline
\includegraphics[width=0.45\columnwidth]{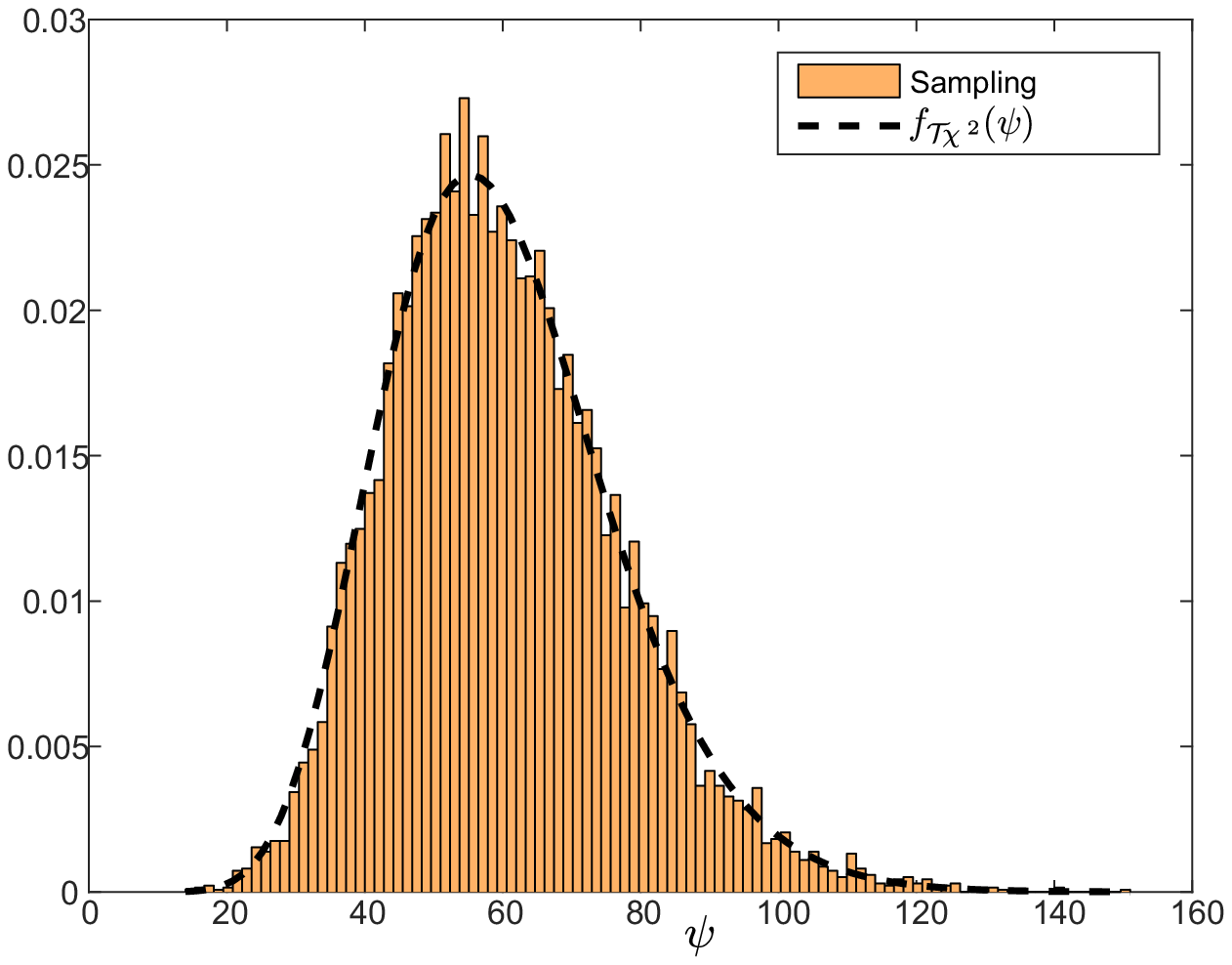} & \includegraphics[width=0.45\columnwidth]{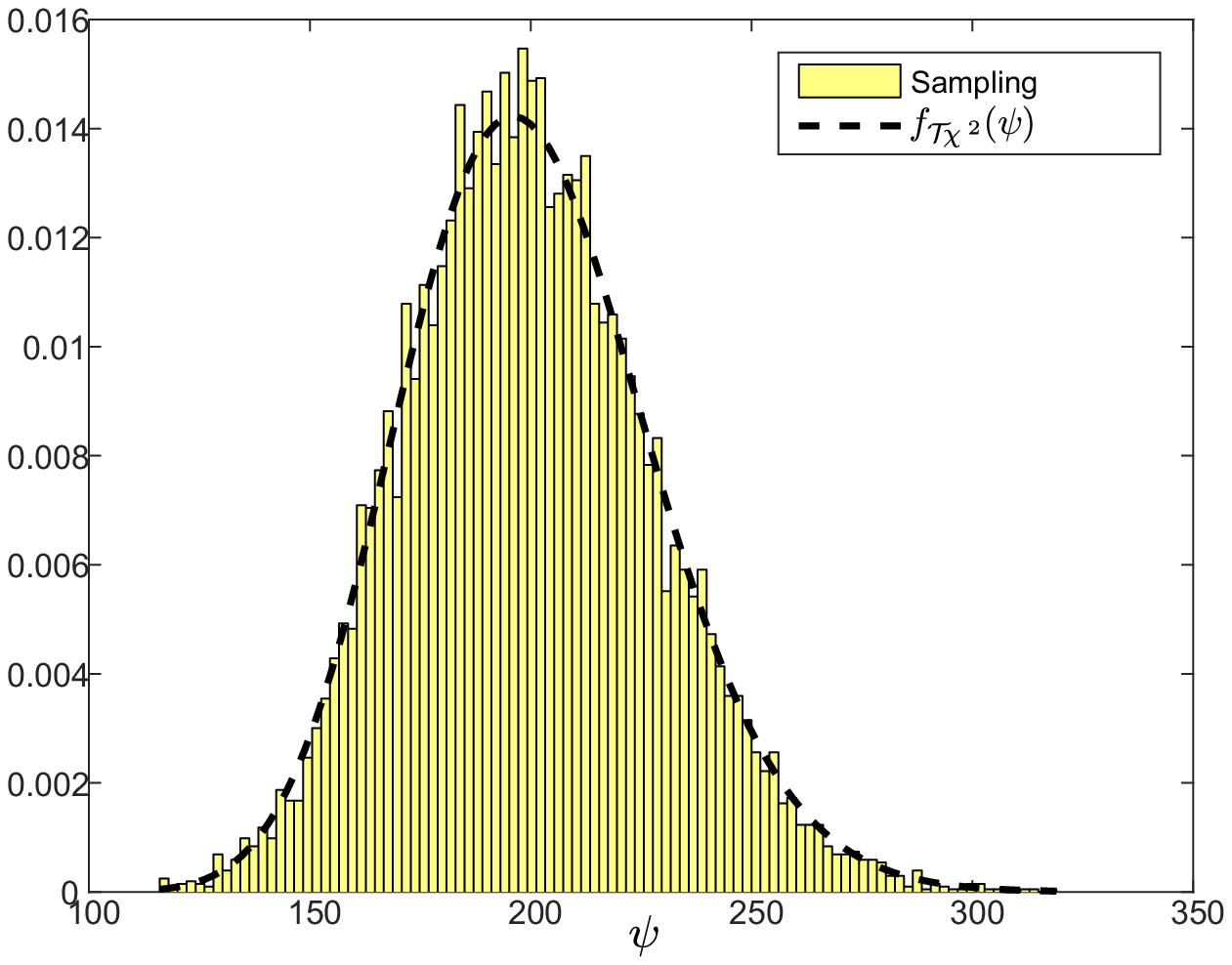}\\
\hline
\hline
\end{tabular}
 \caption{Statistical demonstration of the probability distribution that describe the generation of individuals (offspring) over quadratic basins, considering $\left\{\mathcal{H}_1,~\mathcal{H}_2,~\mathcal{H}_3,~\mathcal{H}_4\right\}$ at $n=\left\{3,~10,~30,~100\right\}$, respectively with the exact forms listed in (H-1)-(H-4).
The distribution of $J(\vec{x})$ is depicted over a single sample of $10,000$ individuals: statistical histograms in bars, versus the analytically derived approximation to the PDF (dashed curve) according to $f_{\mathcal{T}\chi^2}$ in Eq.\ \ref{eq:Nchi_density}. 
\label{table:approxChi2}}
\end{figure}
\begin{table}
\begin{center}
\begin{small}
\begin{tabular}{| l  l |}
\hline
& \\
$ \mathcal{C}^{\textrm{Algo1}}= 
\begin{pmatrix}
0.1552  & -0.0362 &  -0.0096 \\
-0.0362 & 0.1125  &  0.0023\\
-0.0096  & 0.0023 & 0.0766
\end{pmatrix}   $
& 
$ \mathcal{U}^{\textrm{Algo1}}= 
\begin{pmatrix}
0.1606  & -0.4737 & 0.8659 \\
0.0953  & -0.8658  & -0.4913\\
0.9824 &  0.1614 & -0.0939
\end{pmatrix}   $ 
\\
& \\
$ \mathcal{C}^{Eq\ref{eq:Cij}}= 
\begin{pmatrix}
0.1631 &  -0.0369 &  -0.0107 \\
-0.0369  &  0.1188  &  0.0024\\
-0.0107  &  0.0024  &  0.0810
\end{pmatrix} $
&
$ \mathcal{U}^{Eq\ref{eq:Cij}}= 
\begin{pmatrix}
 0.1692 &  -0.4680  &  0.8674 \\
0.0981 & -0.8677 & -0.4873 \\
0.9807  &  0.1675 &  -0.1010
\end{pmatrix} $
\\
& \\
\hline
& \\
$ \mathcal{H}_1 \cdot \mathcal{C}^{Eq\ref{eq:Cij}}= 
\begin{pmatrix}
0.1050  &  0.0038  &  0.0011 \\
0.0039  &  0.1096 &  -0.0003\\
0.0012 &  -0.0003  &  0.1135
\end{pmatrix} $
&
$ \mathcal{U}^{\mathcal{H}_1}= 
\begin{pmatrix}
 0.1692 &  -0.4680  &  0.8674 \\
0.0981 & -0.8677 & -0.4873 \\
0.9807  &  0.1675 &  -0.1010
\end{pmatrix} $
\\
\hline
\end{tabular}
\end{small}
\caption{Numerical integration of Eq.\ \ref{eq:Cij} applied to case (H-1). 
 The statistically-constructed covariance matrix $\mathcal{C}^{\textrm{Algo1}}$ and its corresponding eigenvectors $\mathcal{U}^{\textrm{Algo1}}$ [TOP] \textit{versus} the analytic covariance $ \mathcal{C}^{Eq\ref{eq:Cij}}$ obtained by numerical integration of Eq.\ \ref{eq:Cij} and its eigenvectors $\mathcal{U}^{Eq\ref{eq:Cij}}$ [CENTER]. 
The multiplication $\mathcal{H}_1 \cdot \mathcal{C}^{\textrm{Algo1}}$ is explicitly presented, along with the eigenvectors of the Hessian matrix, $\mathcal{U}^{\mathcal{H}_1}$ [BOTTOM]. \label{table:H1}}
\end{center}
\end{table}

\begin{figure}
$\mathcal{H}_1~(n=3)~\left(\lambda=20, ~ N_{\texttt{iter}}=10^5 \right)$\\
\includegraphics[width=0.49\columnwidth]{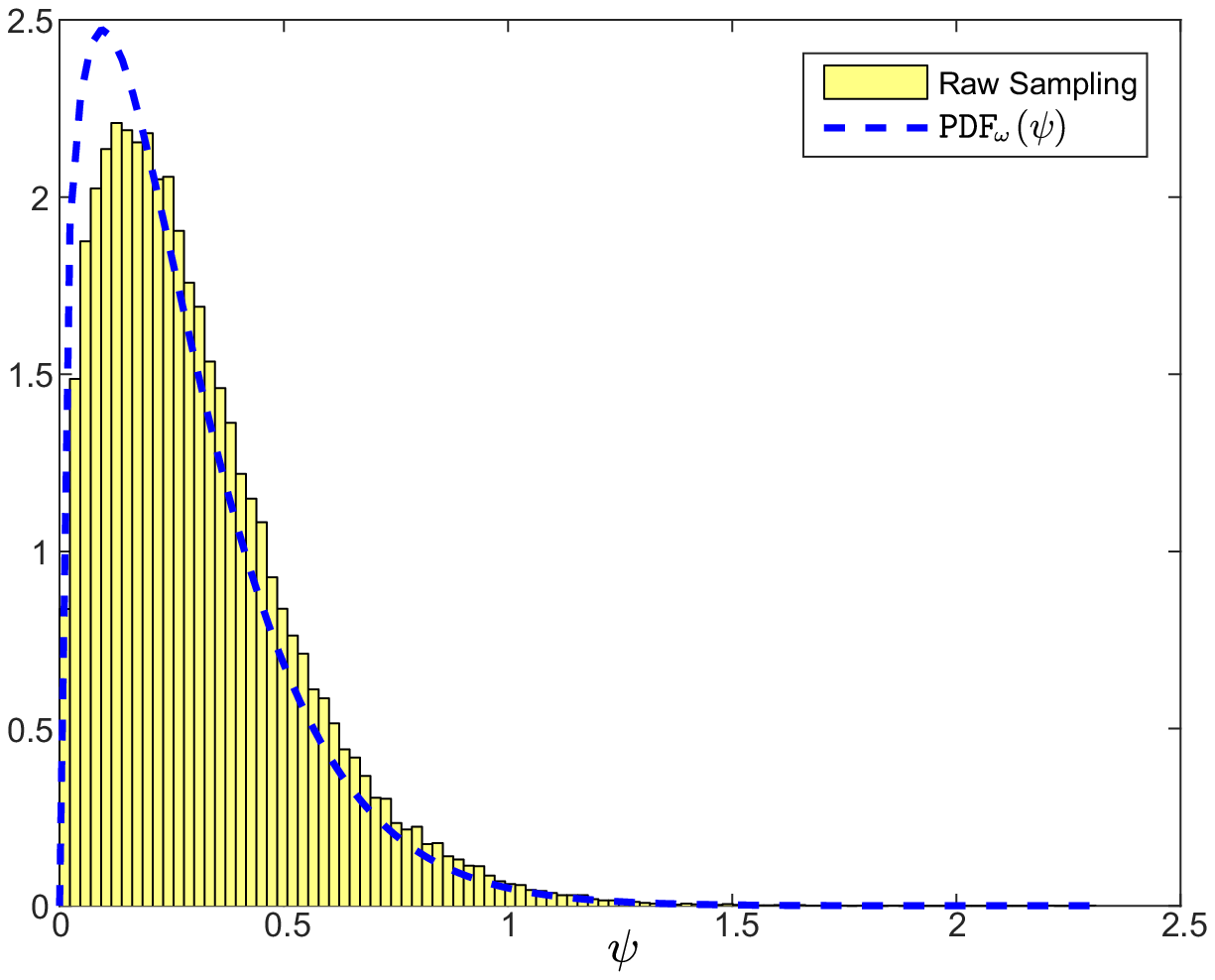}  \includegraphics[width=0.47\columnwidth]{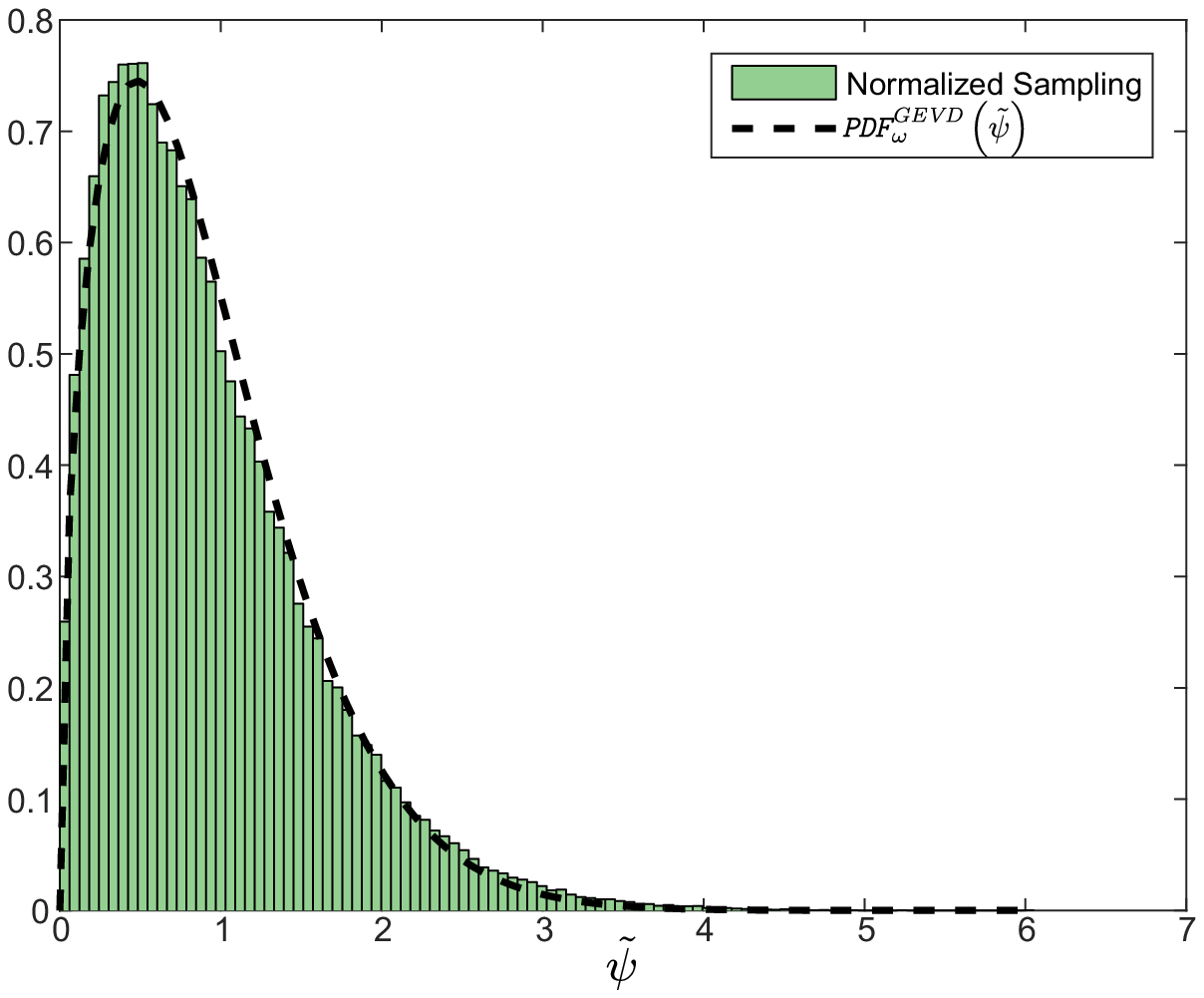} \\
$\mathcal{H}_3~(n=30)~\left(\lambda=1000, ~ N_{\texttt{iter}}=2\cdot 10^5 \right)$\\
\includegraphics[width=0.49\columnwidth]{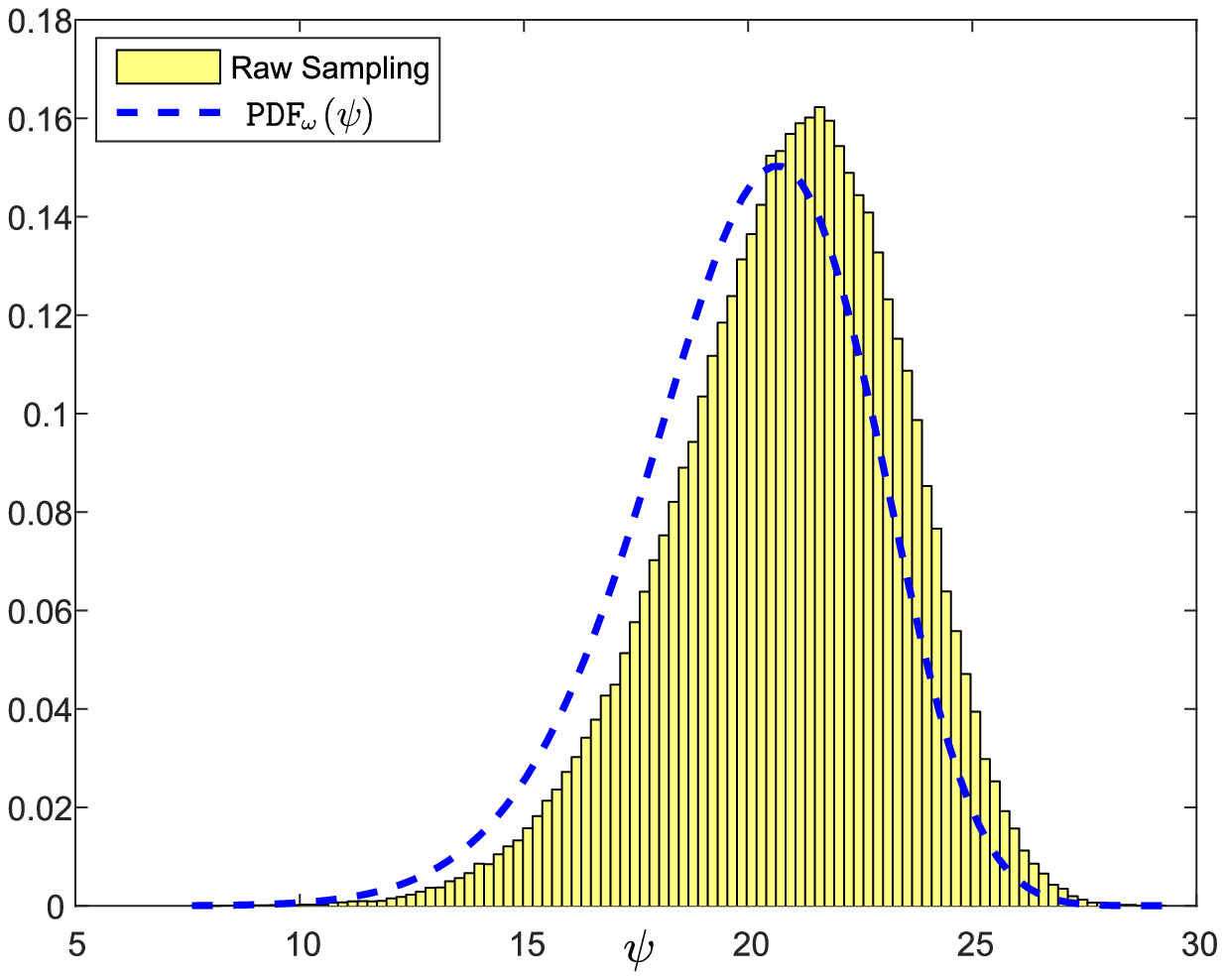}  \includegraphics[width=0.47\columnwidth]{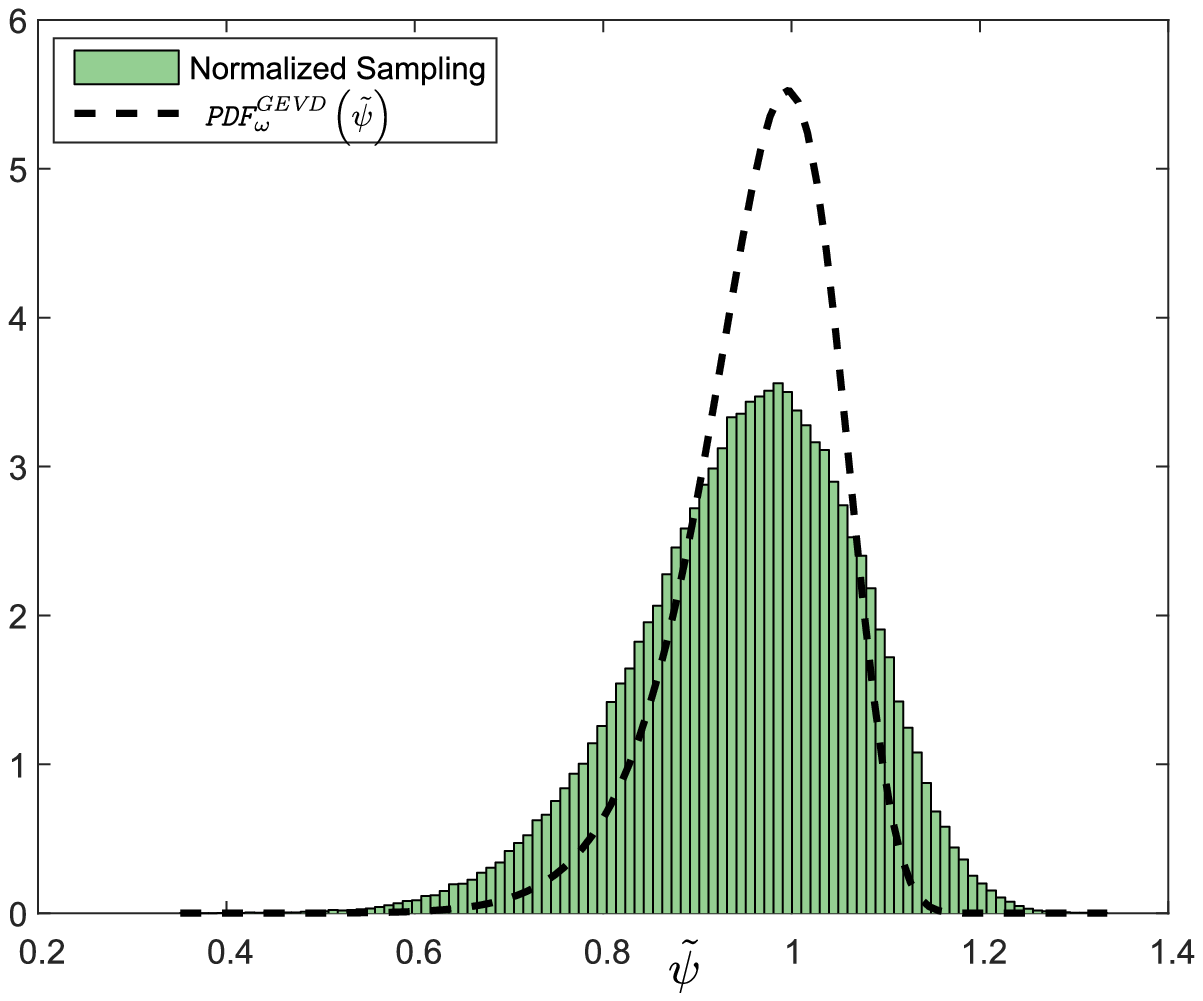} \\
$\mathcal{H}_4~(n=100)~ \left(\lambda=5000, ~ N_{\texttt{iter}}=5\cdot10^5 \right)$ \\
\includegraphics[width=0.49\columnwidth]{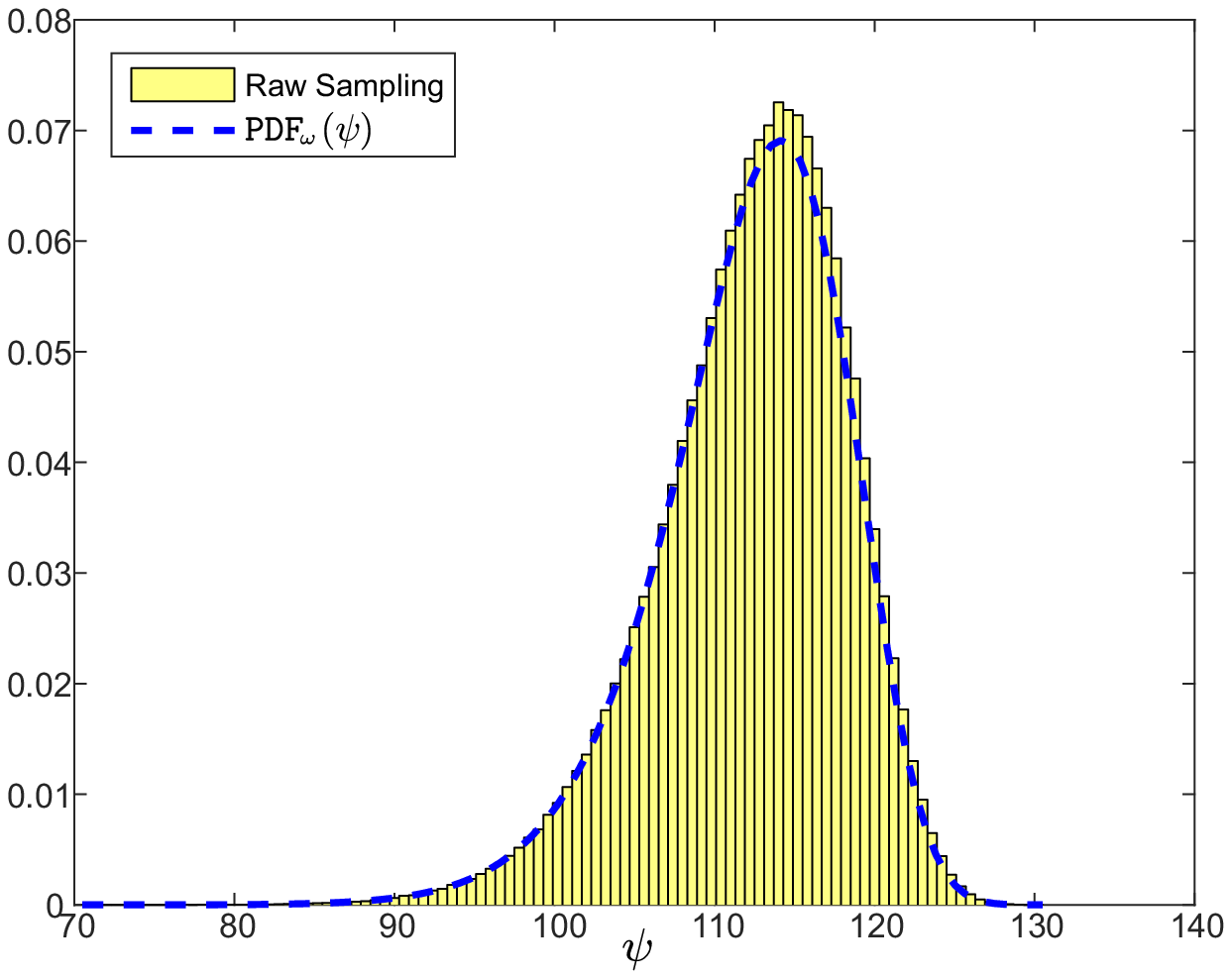} \includegraphics[width=0.47\columnwidth]{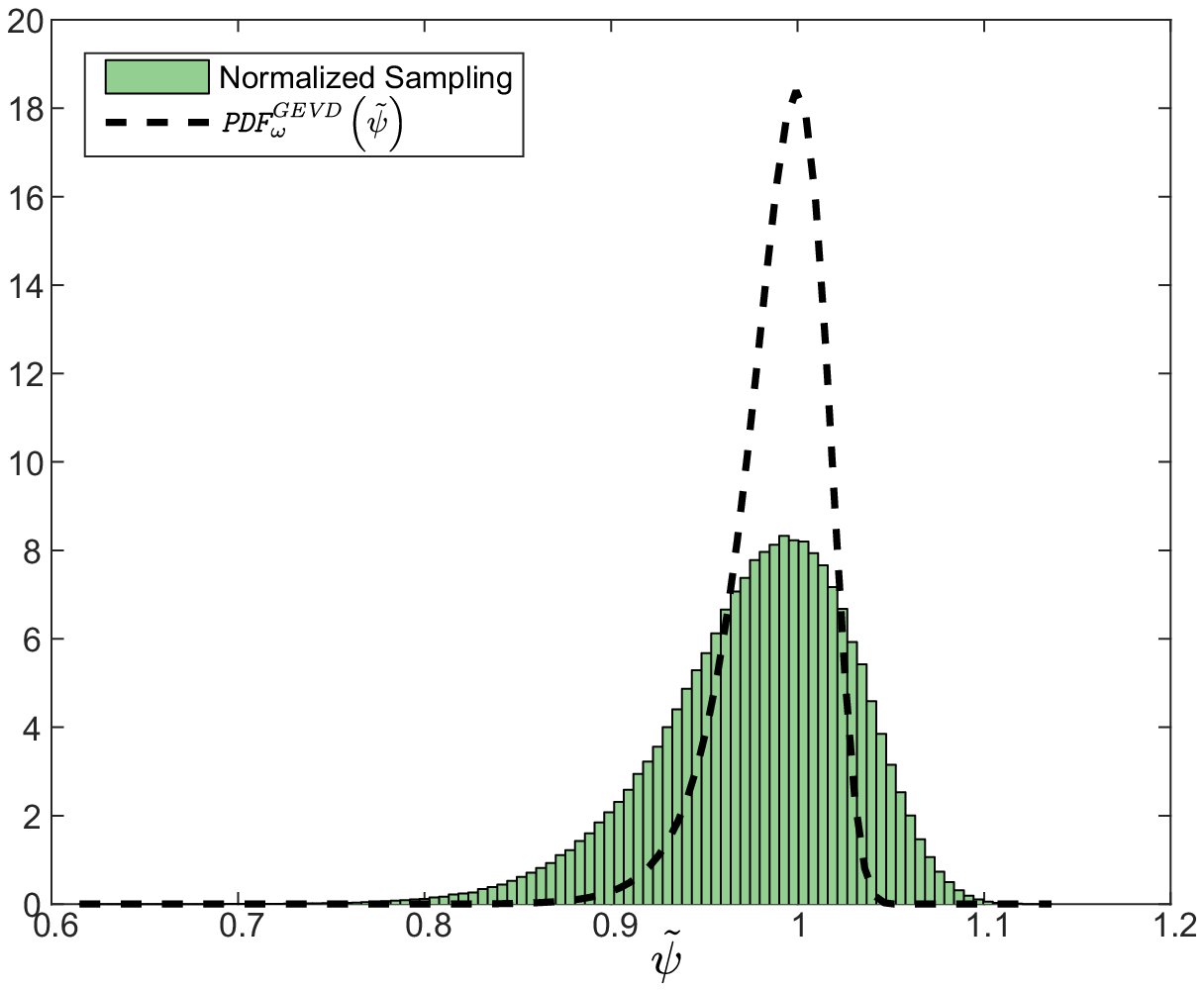} \\
\caption{Statistical demonstration of the derived probability distributions that describe the winning events of quadratic basin minimization, considering $\left\{\mathcal{H}_1,~\mathcal{H}_3,~\mathcal{H}_4\right\}$ at $n=\left\{3,~30,~100\right\}$, respectively with the exact forms listed in (H-1),(H-3),(H-4).
The winners' distributions amongst a population of $\lambda$ individuals are depicted over $N_{\texttt{iter}}$ iterations:  
[LEFT column] Statistical histograms of raw samples (bars), versus the analytical PDF (dashed curve) according to $\texttt{PDF}_{\omega}$ in Eq.\ \ref{eq:y_pdf}; [RIGHT column] Statistical histograms of normalized samples (bars), versus the analytically derived GEVD approximation to the PDF (dashed curve) according to $\texttt{PDF}_{\omega}^{\textrm{GEVD}}$ in Eq.\ \ref{eq:cdf_y_evd}.
\label{table:winners}}
\end{figure}

\begin{figure}
$\mathcal{H}_2~(n=10)~\left(\lambda=1000, ~ N_{\texttt{iter}}=10^5 \right)$\\
\includegraphics[width=0.49\columnwidth]{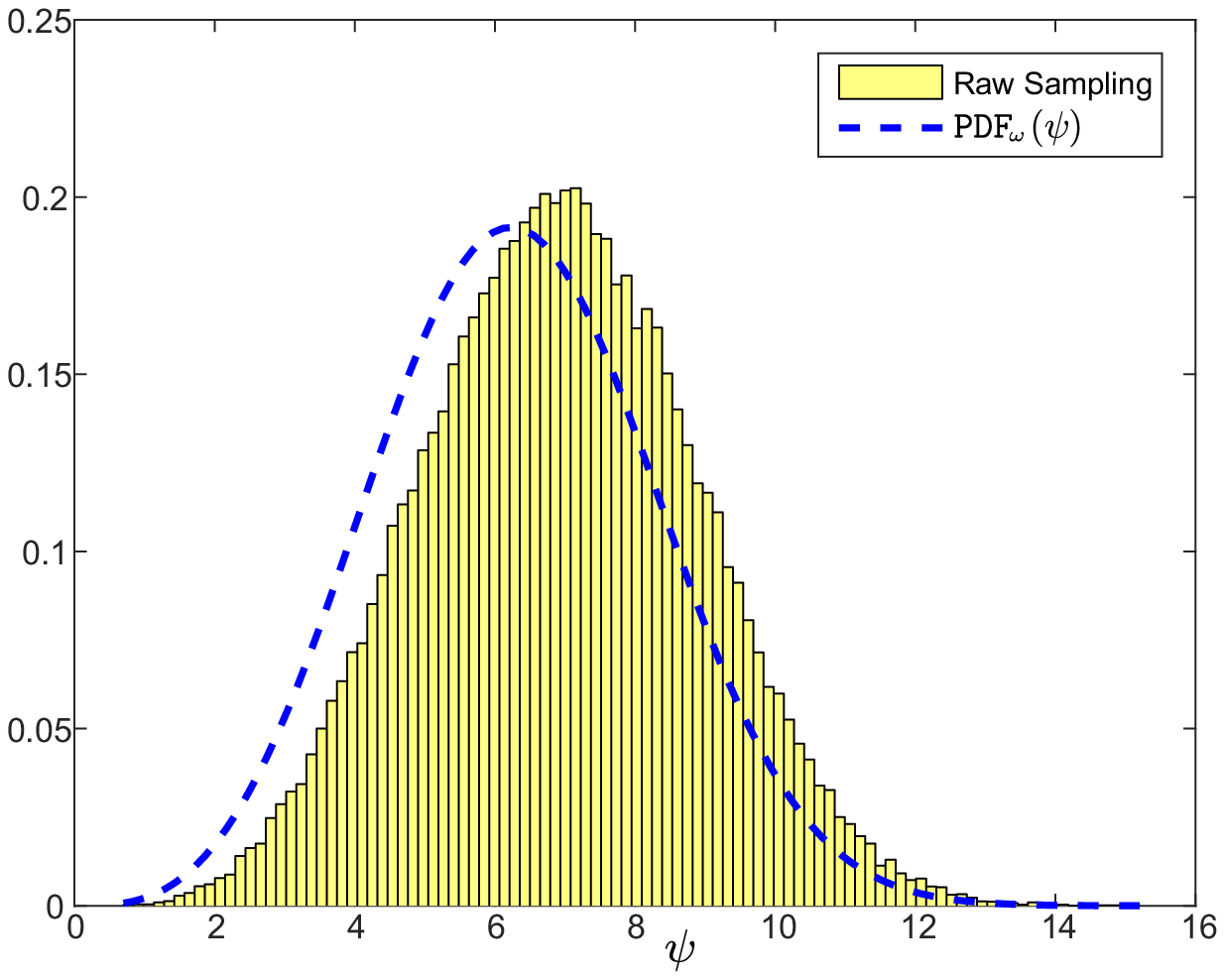}  \includegraphics[width=0.47\columnwidth]{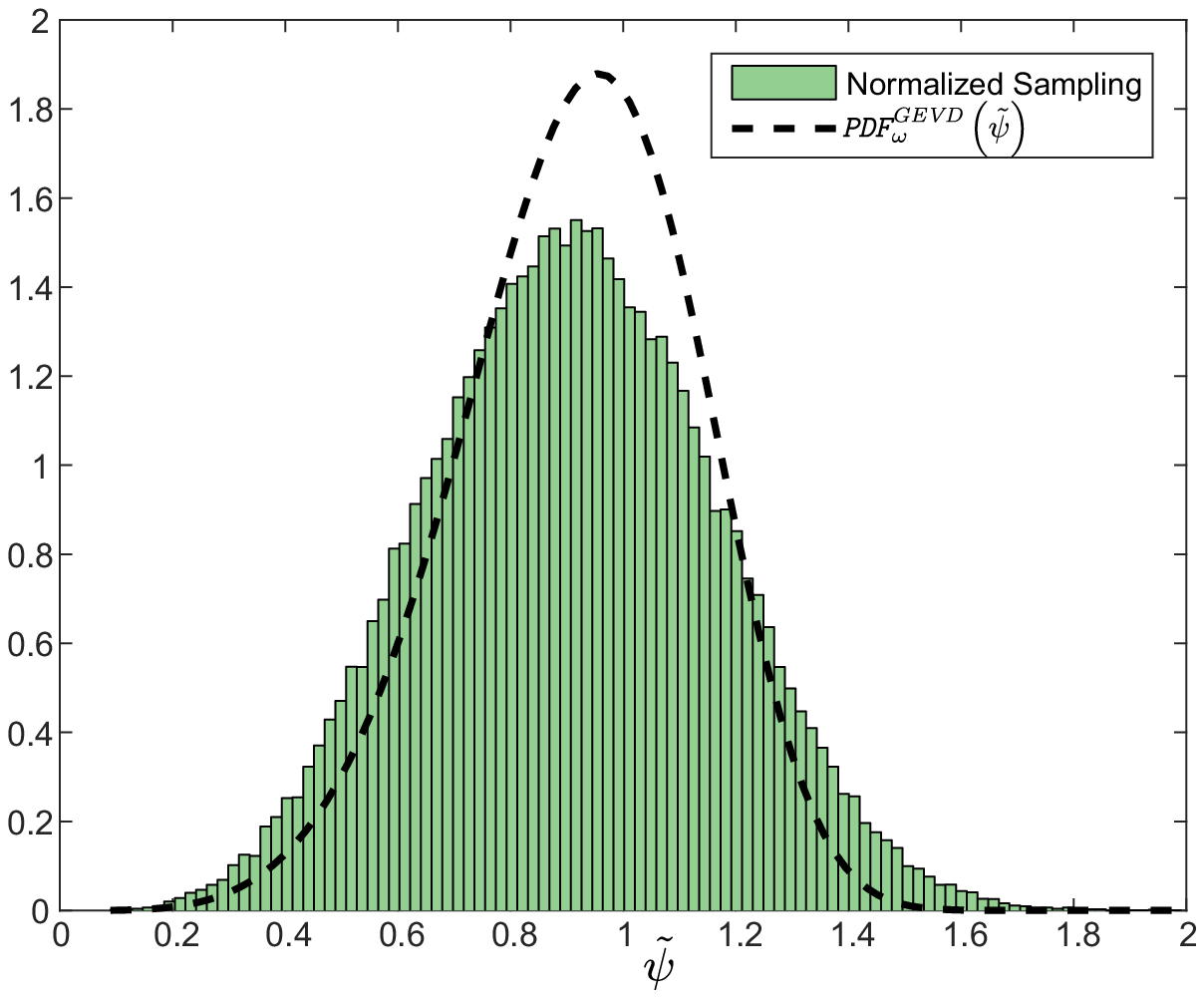} \\
$\mathcal{H}_2~(n=10)~\textrm{[LEFT]:}~ \left(\lambda=10000, ~ N_{\texttt{iter}}=10^6\right)~\textrm{[RIGHT]:}~ \left(\lambda=20000, ~ N_{\texttt{iter}}=5\cdot 10^6\right)$\\
\includegraphics[width=0.5\columnwidth]{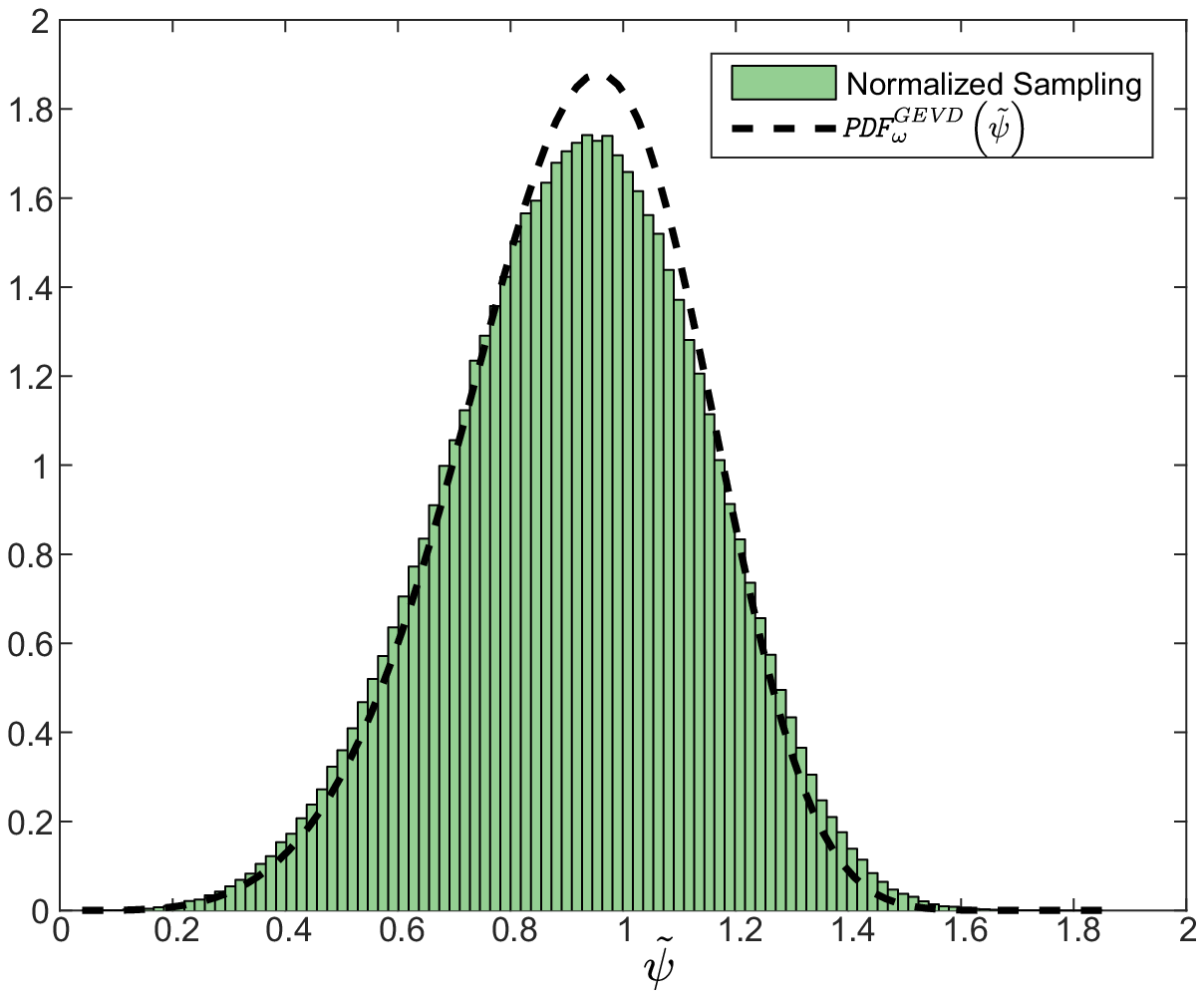} \includegraphics[width=0.5\columnwidth]{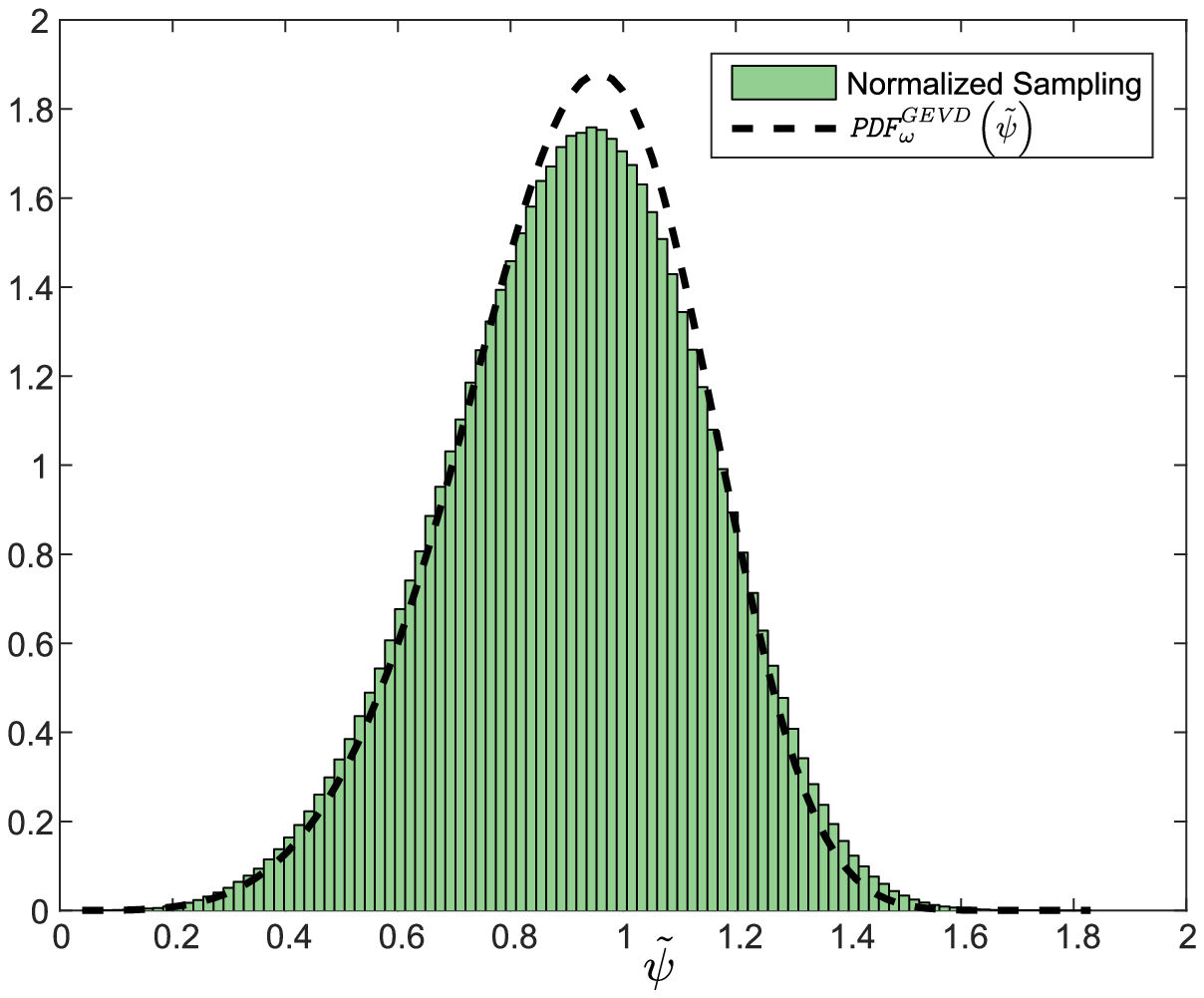} 
 \caption{Statistical demonstration of the distributions describing the winning events of quadratic basin minimization, considering $\mathcal{H}_2$ [TOP], with two additional GEVD approximations for $\mathcal{H}_2$ featuring different settings [BOTTOM]. \label{table:H2winners}}
\end{figure}

\section{Conclusion}\label{sec:discussions}	
Our analytical work modeled passive ES learning, that is, no step-size adaptation nor covariance matrix adaptation were utilized when constructing a covariance matrix from winning decision vectors. 
We proved that the statistically-constructed covariance \textit{commutes} with the landscape Hessian about the optimum when a quadratic basin is assumed. The implication of this result is the enhanced capacity of ESs to identify \textit{sensitive optimization directions} by extracting this information from the learned covariance matrix.
We then derived an analytical approximation for the covariance matrix, based on two presumptions -- (i) the generalized $\chi^2$ density function was approximated, assuming moderate standard deviation of the eigenvalues, and (ii) the winners' distribution was shown to follow the \textit{Weibull} distribution with a calculated \textit{tail index} when the population size $\lambda$ is large, adhering to the limit distributions of order statistics.
Our results were then numerically validated at multiple levels, where the accuracy of the approximations was discussed.

\bibliographystyle{abbrv}

\end{document}